\pdfoutput=1

 \documentclass[10pt,twocolumn,twoside,journal]{IEEEtran}

\usepackage{subfigure}
\usepackage{url}

\usepackage{cite}

\usepackage{graphicx}

\graphicspath{{./}{./Fig/}{./Figs/}}

\usepackage{amsmath}
\usepackage{url}
\usepackage{multirow}

\usepackage{xspace,amssymb}

\newcommand{\ie}{{\it i.e.}{ }}

\newcommand{\etal}{{\it et al.}\xspace}

\newcommand{\0} {{\boldsymbol 0}}
\newcommand{\1} {{\boldsymbol 1}}
\newcommand{\x} {{\boldsymbol x}}

\newcommand{\f} {{\boldsymbol \varsigma}}
\newcommand{\F} {{\boldsymbol f}}

\newcommand{\h} {{\boldsymbol h}}
\renewcommand{\P} {{\bf P}}
\renewcommand{\H} {{\bf  H}}
\newcommand{\q} {{\boldsymbol q}}
\renewcommand{\u} {{\boldsymbol u}}
\renewcommand{\d} {{\boldsymbol d}}
\renewcommand{\v} {{\boldsymbol v}}
\newcommand{\w} {{\boldsymbol \omega}}
\newcommand{\gammav}{{\bm\gamma}}
\newcommand{\rhos} {{\bm \rho}}
\newcommand{\xis}{{\bm \xi}}

\newcommand{\wt} {{\omega^{(t)}}}
\newcommand{\wj} {{\omega^{(j)}}}

\newcommand{\M} {{\mathbf{M}}}

\def\Xset{\mathbb{X}}
\def\Yset{\mathbb{Y}}
\def\Fset{\mathbf{F}}
\def\Rset{\mathbb{R}}
\newcommand{\dom} {{\textbf{dom}}}

\newcommand{\wone} {{\| \boldsymbol \omega \|_1}}
\newcommand{\cbutyi}{{c = 1 \atop c \neq y_i}}
\newcommand{\cbey}{{c \in \mathscr{Y} \atop c \neq y_i}}
\newcommand{\sumN} {{\sum_{i=1}^N}}
\newcommand{\sumT} {{\sum_{t=1}^T}}
\newcommand{\sumNK} {{\sum_{i=1}^N \sum_{k=1}^K}}
\newcommand{\sumNL} {{\sum_{i=1}^N \sum_{l=1}^L}}
\newcommand{\sumNC} {{\sum_{i=1}^N \sum_{c=1}^C}}
\newcommand{\sumNCR}{{\sum_{i=1}^N \sum_{\cbutyi}^C}}
\newcommand{\dtci}{{\delta_{c,y_i}}}

\newcommand\T{{\!\top}}
\newcommand{\st}    {{\mathrm{s.t.}}}
\newcommand{\argmax}{\mathop{\rm argmax}}
\newcommand{\Lone}  {{\ell_1}}
\newcommand{\Th}    {{\text{-}{\rm th}}}

\newcommand\ind{\mathbbm{1}}

\newcommand{\Input} {\textbf{Input}}
\newcommand{\outpt}{\textbf{Output}}
\newcommand{\Init} {\textbf{Initialization}}
\newcommand{\If} {\textbf{if}{ }}
\newcommand{\then} {\textbf{then}{ }}
\newcommand{\tab}   {\hspace{5.4 mm}} %
\newcommand{\btab}  {\hspace{9.6 mm}} 

\newcommand{\MultiA}{MultiBoost.MO\xspace}
\newcommand{\MultiB}{MultiBoost.ECC\xspace}

\newcommand{\thy}  {{\emph{thyroid}}}

\newcommand{\dna}  {{\emph{dna}}}

\newcommand{\svf}  {{\emph{svmguide4}}}
\newcommand{\win}  {{\emph{wine}}}
\newcommand{\iri}  {{\emph{iris}}}
\newcommand{\veh}  {{\emph{vehicle}}}
\newcommand{\gla}  {{\emph{glass}}}
\newcommand{\svt}  {{\emph{svmguide2}}}

\newcommand{\comment}[1]{}

\usepackage{mathrsfs}
\usepackage{bm}
\usepackage{algorithm}
\usepackage{algorithmic}
\usepackage{booktabs}

\usepackage{bbm}

\newtheorem{theorem}{Theorem}

\begin{document}

\title{ Totally Corrective Multiclass Boosting
        \\ with Binary Weak Learners}

\author{
         Zhihui Hao,
         Chunhua~Shen,
         Nick Barnes,
         and
         Bo Wang
\thanks
{
Z. Hao and B. Wang are with Beijing Institute of Technology, Beijing 100081, China.
(e-mail: hzhbit@gmail.com and wangbo@bit.edu.cn),
Z. Hao's contribution was made when he was visiting
NICTA Canberra Research Laboratory
and the Australian National University.
}
\thanks
{
C. Shen and N. Barnes are with NICTA, Canberra Research Laboratory,
Canberra, ACT 2601, Australia,
and also with the Australian National University, Canberra,
ACT 0200, Australia
(e-mail: chunhua.shen@nicta.com.au, nick.barnes@nicta.com.au).
Correspondence should be addressed to C. Shen.
}

\thanks
{
NICTA is funded by the Australian Government as represented by the
Department of Broadband, Communications and the Digital Economy and the
Australian Research Council through the ICT Center of Excellence program.
This work was also supported in part by the Australian Research
Council through its special research initiative in bionic vision science and
technology grant to Bionic Vision Australia.
}
}

\markboth{September 2010}
{   Hao \MakeLowercase{\textit{et al.}}:
    Totally Corrective Multiclass Boosting with Binary Weak Learners
}

\maketitle

\begin{abstract}

    In this work,  we propose a new optimization framework
    for multiclass boosting learning.
    In the literature,
    AdaBoost.MO and AdaBoost.ECC are the two successful multiclass boosting
    algorithms, which can use binary weak learners.
    We explicitly derive these two algorithms' Lagrange dual problems
    based on their regularized loss functions.
    We show that the Lagrange dual formulations enable us to
    design totally-corrective multiclass
    algorithms by using the primal-dual optimization technique.
    Experiments on benchmark data sets suggest that
    our multiclass boosting
    can achieve a comparable generalization capability with state-of-the-art,
    but the convergence speed is much faster than stage-wise gradient descent boosting.
    In other words, the new totally corrective algorithms
    can maximize the margin more aggressively.

\end{abstract}


\section{Introduction} \label{Sect:Intro}

    Boosting is a powerful learning technique
    for improving the accuracy of any given classification algorithm.
    It has been attracting much research interest in the machine learning
    and pattern recognition community.
    Since Viola and Jones applied boosting to face detection \cite{Viola2004Robust},
    it has shown great success in computer vision,
    including the applications of object detection and tracking
    \cite{ong2004boosted,okuma2004boosted}
    generic object recognition \cite{opelt2006generic},
    image classification \cite{dollar2007feature}
    and retrieval \cite{tieu2004boosting}.

    The essential idea of boosting is to find a combination of \emph{weak} hypotheses
    generated by a base learning oracle.
    The learned ensemble is called the \emph{strong} classifier
    in the sense that it often achieves a much higher accuracy.
    One of the most popular boosting algorithms is AdaBoost \cite{schapire1999improved},
    which has been proven  a method of minimizing
    the regularized exponential loss function \cite{schapire1999improved,shen2010dual}.
    There are many variations on AdaBoost in the literature.
    For example, LogitBoost \cite{hastie2000additive},
    optimizes the logistic regression loss instead of the exponential loss.
    To understand how boosting works,
    Schapire \etal \cite{schapire1998boosting} introduced the margin theory and
    suggested that
    boosting is especially effective at maximizing the margins of training examples.
    Based on this concept, Demiriz \etal \cite{demiriz2002linear} proposed LPBoost,
    which maximizes the minimum margin using the hinge loss.

    Since most of the pattern classification problems in real world are multiclass
    problems, researchers have extended binary boosting algorithms to
    the multiclass case.
    For example in  \cite{Freund1996New}, Freund and Schapire have described
    two possible extensions of AdaBoost to the multiclass case.
    AdaBoost.M1 is the first and perhaps the most direct extension.
    In AdaBoost.M1, a weak hypothesis assigns only one of $C$ possible labels to each instance.
    Consequently, the requirement for weak hypotheses that training error
    must be less than $1/2$ becomes harder to achieve,
    since random guessing only has an accuracy rate of $1/C$ in multiclass case.
    To overcome this difficulty, AdaBoost.M2 introduced a relaxed error measurement
    termed \emph{pseudo-loss}. In AdaBoost.M2,
    the weak hypothesis is required  to answer $C-1$ questions
    for one training example $(\x_i, y_i)$:
    which is the label of $\x_i$, $ c $ or $ y_i $ ($\forall c \neq y_i$)?
    A falsely matched pair $(\x_i, c)$ is called a \emph{mislabel}.
    Pseudo-loss is defined as the weighted average of the probabilities of all incorrect answers.
    Recently, Zhu \etal \cite{zhu2006multi} proposed a multiclass exponential loss function.
    Boosting algorithms based on this loss, including SAMME \cite{zhu2006multi}
    and GAMBLE \cite{huang2007efficient}
    only require the weak hypothesis performs better than random guessing ($1/C$).

    The above-mentioned multiclass boosting algorithms  have a common property:
    the employed weak hypotheses  should have the ability
    to give predictions on all $C$ possible labels at each call.
    Some powerful weak learning methods may be competent, like decision trees.
    However, they are complicated and time-consuming for training compared with binary learners.
    A higher complexity of assembled classifier often implies
    a larger risk of over-fitting the training data
    and possible decreasing of the generalization ability.

    Therefore, it is natural to put forward another idea:
    if a multiclass problem can be reduced into multiple two-class ones,
    binary weak learning method such as linear discriminant analysis
    \cite{skurichina2002bagging, masip2005boosted},
    decision stump or product of decision stumps \cite{kegl2009boosting}
    might be applicable to these decomposed subproblems.
    To make the reduction, one has to introduce some appropriate coding strategy
    to translate each label to a fixed binary string,
    which is usually referred to as a \emph{codeword}.
    Then weak hypotheses can be trained at every bit position.
    For a test example, the label is predicted by decoding the codeword
    computed from the strong classifier.
    %
    %
    %
    AdaBoost.MO \cite{schapire1999improved} is a representative algorithm
    with this coding-decoding process.
    To increase the distance between codewords and thus improve the error correcting ability,
    Dietterich and Bakiri's error-correcting output codes (ECOC) \cite{dietterich1995solving}
    can be used in AdaBoost.MO.
    A variant of AdaBoost.MO, AdaBoost.OC, also combines boosting and ECOC.
    However, unlike AdaBoost.MO, AdaBoost.OC employs a collection
    of randomly generated codewords.
    For more details about the random methods, we refer the reader to \cite{schapire1997using}.
    AdaBoost.OC penalizes both the wrongly classified examples
    and the mislabels in correctly classified examples by calculating pseudo-loss.
    Hence, AdaBoost.OC may be viewed as a special case of AdaBoost.M2.
    The difference is that, weak hypotheses in AdaBoost.OC are required to answer one binary question
    for each training example: which is the label for $\x_i$ in the current round, $0$ or $1$?
    Later, Guruswami and Sahai \cite{guruswami1999multiclass} proposed a variant, AdaBoost.ECC,
    which replaces pseudo-loss with the common measurement to compute training errors.

    In this work, we mainly focus on the multiclass boosting algorithms
    with {\em binary weak learners}. Specifically, AdaBoost.MO and AdaBoost.ECC.
    It has been proven AdaBoost.OC is in fact a shrinkage version of AdaBoost.ECC
    \cite{sun2007unifying}.
    These two algorithms both perform stage-wise functional gradient descent procedures
    on the exponential loss function \cite{sun2007unifying}.
    In \cite{shen2010dual}, Shen and Li have shown that $\Lone$ norm regularized
    boosting algorithms including AdaBoost, LogitBoost and LPBoost,
    might be optimized through optimizing their corresponding dual problems.
    This primal-dual optimization technique is also {\em implicitly}
    applied by other studies
    to explore the principles of boosting learning
    \cite{kivinen1999boosting,crammer2001algorithmic}.
    Here we study AdaBoost.MO and AdaBoost.ECC
    and explicitly derive their Lagrange dual problems.
    Based on the primal-dual pairs,
    we put these two algorithms into a column generation based
    primal-dual optimization framework.
    We also analytically show the boosting algorithms that we proposed
    are {\em totally corrective} in a relaxed fashion.
    Therefore, the proposed algorithms seem to  converge more effectively
    and  maximize the margin of training examples more aggressively.
    To our knowledge, our proposed algorithms are the first totally corrective
    multiclass boosting algorithms.

    The notation  used in this paper is as follows.
    We use the symbol $\M$ to denote a coding matrix with $M(a,b)$ being its $(a,b)$-th entry.
    Bold letters ($\u, \v$) denote column vectors.
    $\0$ and $\1$ are vectors with all entries being $0$ and $1$ respectively.
    The inner product of two column vectors $\u$ and $\v$
    are expressed as $\u^\T \v = \sum_i u_i v_i$.
    Symbols $\succeq$, $\preceq$ placed between two vectors indicate
    that the inequality relationship holds for all pairwise entries.
    Double-barred letters ($\Rset$, $\Yset$) denote  specific domains or sets.
    The abbreviation $\st$ means ``subject to".
    $\ind (\pi)$ denotes an indicator function
    which gives $1$ if $\pi$ is true and $0$ otherwise.

    The remaining content  is organized as follows.
    In Section \ref{sect:algs} we briefly review the coding strategies
    in multiclass boosting learning
    and describe the algorithms of AdaBoost.MO and AdaBoost.ECC.
    Then in Section \ref{Sect:AlgsTC} we derive the primal-dual relations
    and propose our multiclass boosting learning framework.
    In Section \ref{Sect:Exps} we compare the related algorithms through several experiments.
    We conclude the paper in Section \ref{Sect:cons}.

    \begin{algorithm}[!ht]
    \caption{AdaBoost.MO (Schapire and Singer, 1999)}
    \label{alg:mo}
    \begin{algorithmic}
    \STATE
        \Input\ training data $ (\x_i, y_i)$, $y_i \in \{1,\dots, C\}$, $i =1,\dots, N$;\\
        \btab   maximum training iterations $T$, and the coding matrix $\M^{C\times L}$.
    \STATE
        \Init\\
        \tab   Weight distribution\\
        \tab    $u_{i,l} = \frac{1}{NL}$, $ i = 1,\dots, N $, $l = 1,\dots, L$. \\
    \FOR{$t = 1$ : $T$ }
    \STATE
        a) Normalize $\u$;
    \STATE
        b) Train $L$ weak hypotheses $h_l^{(t)}(\cdot)$ according to the weight distribution $\u$;
    \STATE
        c) Compute $\epsilon = \sum_i \sum_l u_{i,l} \ind (M(y_i, l) \neq h_l^{(t)}(\x_i))$;
    \STATE
        d) Compute $\wt = \frac{1}{2} \ln (\frac{1-\epsilon}{\epsilon})$;
    \STATE
        e) Update $u_{i,l} = u_{i,l} \exp \bigl( - \wt M(y_i, l)h_l^{(t)}(\x_i)
                                          \bigr)$;
    \ENDFOR
    \STATE
        \outpt\ $\F(\cdot)=\bigl[
          \sum_t\wt h_1^{(t)}(\cdot), \cdots, \sum_t\wt h_L^{(t)}(\cdot)
          \bigr]^\T$.
    \end{algorithmic}
    \end{algorithm}

\section{Multiclass boosting algorithms and coding matrix}
\label{sect:algs}

    In this section, we briefly review the multiclass boosting algorithms of
    AdaBoost.MO \cite{schapire1999improved}
    and AdaBoost.ECC \cite{guruswami1999multiclass}.

    A typical multiclass classification problem can be expressed as follows.
    A training set for learning is given by $\{(\x_i, y_i)\}_{i=1}^N$.
    Here $ \x_i $ is a pattern and $ y_i $ is the label,
    which takes a value from the space $\Yset = \{1,2,\dots, C \}$ if we have $C$ classes.
    The goal of classification is then
    to find a classifier $ \F: \Xset \rightarrow \Yset $
    which assigns one and only one label to a new observation $(\x, y)$
    with a minimal probability of $ y \neq \F(\x) $.
    Boosting algorithm tries to find an ensemble function in the form of
    $\F(\x) = \sumT \wt h^{(t)}(\x)$ (or equivalently the normalized version
    $ \sumT \F(\x)/ \sum_t \wt$),
    where $h(\cdot)$ denotes the weak hypotheses generated by base learning algorithm
    and $\w = [\omega^{(1)} \cdots \omega^{(T)}]^\T$ denotes the associated coefficient vector.
    Typically, a weight distribution $\u$ is used on training data, which essentially makes the
    learning algorithm
    concentrate on those examples that are hard to distinguish.
    The weighted training error of hypothesis $h(\cdot)$ on $\u$
    is given by $\sum_i u_i \ind(y_i \neq h(\x_i))$.


    To decompose a multiclass problem into several binary subproblems,
    a coding matrix $\M \in \{\pm 1\}^{C\times L}$ is required.
    Let $M(c,:)$ denote the $c\Th$ row, which represents a $L$-length codeword for class $c$.
    One binary hypothesis can be learned then for each column,
    where training examples has been relabeled into two classes.
    For a newly observed instance $\x$, $\F(\x)$ outputs an unknown codeword.
    Hamming distance or some loss-based measure is used to calculate the distances
    between this word and rows in $\M$.
    The ``closest" row is identified as the predicted label.
    For binary strings, loss-based measures are equivalent to Hamming distance.

    \begin{algorithm}
    \caption{AdaBoost.ECC (Guruswami and Sahai, 1999)}
    \label{alg:ecc}
    \begin{algorithmic}
    \STATE
        \Input\ training data $ (\x_i, y_i)$, $y_i \in \{1,\dots, C\}$, $i =1,\dots, N$;\\
        \btab   maximum training iterations $T$.
    \STATE
        \Init\\
        \tab The coding matrix $\M = [{\over}]$, and   the weight distribution \\
        \tab    $u_{i,c} = \frac{1}{N(C-1)}$,$ i = 1,\dots, N $, $c = 1,\dots, C$. \\
    \FOR{$t = 1$ : $T$ }
    \STATE
        a) Create $M(:, t) \in \{ -1, +1 \}^{C \times 1}$;
    \STATE
        b) Normalize $\u$;
    \STATE
        c) Compute weight distribution for mislabels \\
        \tab    $d_i = \sum_c u_{i,c} \ind (M(c, t) \neq M(y_i, t)) $;
    \STATE
        d) Normalize $\d$;
    \STATE
        e) Train a weak hypothesis $h^{(t)}(\cdot)$ using $\d$;
    \STATE
        f) Compute $\epsilon = \sum_i d_i \ind (M(y_i, t) \neq h^{(t)}(\x_i))$;
    \STATE
        g) Compute $\wt = \frac{1}{4} \ln (\frac{1-\epsilon}{\epsilon})$;
    \STATE
        h) Update $u_{i,c} = $\\
        \tab    $u_{i,c} \exp
                \left( - \wt \left(M(y_i, t) - M(c,t) \right) h^{(t)}(\x_i)\right)$;
    \ENDFOR
    \STATE
        \outpt\ $\F(\cdot)=[\omega^{(1)} h^{(1)}(\cdot), \cdots,
                    \omega^{(T)} h^{(T)}(\cdot)]^\T$.
    \end{algorithmic}
    \end{algorithm}

    The coding process can be viewed as a mapping to a new higher dimensional space.
    If the codewords in the new space are mutually distant,
    the more powerful error-correction ability can be gained.
    For this reason, Dietterich and Bakiri's error-correcting output codes (ECOC)
    are usually chosen. Especially if the coding matrix equals to the unit matrix
    (\ie, each codeword is one basis vector in the high dimensional space),
    no error code could be corrected.
    This is the \emph{one-against-all} or \emph{one-per-class} approach.
    Several coding matrices have been evaluated in \cite{allwein2001reducing}.
    Another family of codes are random codes
    \cite{schapire1997using, guruswami1999multiclass, crammer2002learnability}.
    Compared with fixed codes, random codes are more flexible to explore the relationships
    between different classes.

    %
    %
    AdaBoost.MO employs a predefined coding matrix, such as ECOC.
    An $L$-dimensional weak hypothesis is trained at each iteration,
    with one entry for a binary subproblem.
    The pseudo-code of AdaBoost.MO is given in Algorithm \ref{alg:mo}.
    For a new observation, the label can be predicted by
    \begin{eqnarray}
    y
    \nonumber
    & = &
        \arg \max_{c=1}^C \left\{ M(c,:) \F(\x) \right\}
    \\
    \label{EQ:MO-predict}
    & = &
        \arg \max_c \{ \sum_{j=1}^T \wj M(c,:) \h^{(j)}(\x) \}.
    \end{eqnarray}
    For AdaBoost.ECC, an incremental coding matrix is involved.
    At each iteration, a randomly generated code is added into the matrix,
    which corresponds a new binary subproblem being created.
    AdaBoost.ECC is summarized in Algorithm \ref{alg:ecc}.
    The prediction is implemented by
    \begin{eqnarray}
    y
    \label{EQ:ECC-predict}
    & = &
        \arg \max_c \{ \sum_{j=1}^T \wj M(c,j) h^{(j)}(\x) \}.
    \end{eqnarray}
    Comparing \eqref{EQ:ECC-predict} with \eqref{EQ:MO-predict},
    we can see AdaBoost.ECC is similar to AdaBoost.MO in terms of coding-decoding process.
    Roughly speaking, the coding matrix in ECC is degenerated into a random changing column;
    accordingly, induced binary problems turn to be a single one.
    The relationship between these two algorithms will be completely clear
    after we derive the dual problems in the next section.

\section{Totally corrective multiclass boosting} \label{Sect:AlgsTC}

    In this section we present the $\Lone$ norm regularized optimization problems
    that AdaBoost.MO and AdaBoost.ECC solve,
    and derive the corresponding Lagrange dual problems.
    Based on the column generation technique,
    we design new boosting methods for multiclass learning.
    Unlike conventional boosting, the proposed algorithms are totally corrective.
    For ease of exposition,
    we name our algorithms as \MultiA{} and \MultiB.

\subsection{Lagrange Dual of AdaBoost.MO}
    The loss function that AdaBoost.MO optimizes is proposed in \cite{sun2007unifying}.
    Given a coding strategy $\Yset \rightarrow \M^{C\times L}$,
    the loss function can be written as
    \begin{eqnarray}
    \label{EQ:MO-loss}
        L_{\rm MO} = \sumNL \exp \left( -M(y_i, l) F_l(x_i) \right),
    \end{eqnarray}
    where $F_l(x_i)$ is the $l\Th$ entry in the strong classifier
    and $F_l(x_i) = \sumT \wt h_l^{(t)}(x_i)$.
    For a training example $\x_i$,
    let $\H_l(\x_i) = [h_l^{(1)}(\x_i) \ h_l^{(2)}(\x_i) \cdots h_l^{(T)}(\x_i)]^\T$
    denote the outputs of the $l\Th$ weak hypothesis at all $T$ iterations.
    The boosting process of AdaBoost.MO is equivalent to
    solve the following optimization problem:
    \begin{eqnarray}
      \label{EQ:MO-prim}
        &\min\limits_{\w} &
         \sumNL \exp \left( - M(y_i, l) \H_l^\T(\x_i) \w \right)
        \\
      \label{EQ:MO-prim-cons}
        & \st &
        \w \succeq \0, \wone = \theta.
    \end{eqnarray}
    Notice that the constraint $\wone = \theta$ with $\theta > 0$
    is not explicitly enforced in AdaBoost.MO.
    However, if the variable $\w$ is not bounded,
    one can arbitrarily make the loss function approach zero
    via enlarging $\w$ by an adequately large factor.
    For a convex and monotonically increasing function,
    $\wone = \theta$ is actually equivalent to $\wone \leq \theta$
    since $\w$ always locates at the boundary of the feasibility set.

    By using Lagrange multipliers,
    we are able to derive the Lagrange dual problem
    of any optimization problem \cite{boyd2004convex}.
    If the strong duality holds,
    the optimal dual value is exactly the same as the optimal value of primal problem.

    \begin{theorem}
    \label{TM:Dual-MO}
        The Lagrange dual problem of program \eqref{EQ:MO-prim} is
        \begin{eqnarray}
        \label{EQ:MO-dual}
          & \max\limits_{r,\u}&
           -r\theta
           - \sumNL u_{i,l}\log u_{i,l}
           + \1^\T \u
          \\
        \nonumber
          & \st &
           \sumNL u_{i,l} M(y_i, l) \H_l^\T(\x_i) \preceq r \1^\T,
           \u \succeq \0.
        \end{eqnarray}
    \end{theorem}
    \begin{proof}
        To derive this Lagrange dual,
        one needs to introduce a set of auxiliary variables
        $ \gamma_{i,l} = - M(y_i, l) \H_l^\T(\x_i)\w $,
        $ i = 1,\dots, N $, $ l = 1,\dots, L$.
        Then we can rewrite the primal \eqref{EQ:MO-prim} into
        \begin{eqnarray}
        \label{PF:MO-prim}
            &\min\limits_{\w} &
             \sumNL \exp \gamma_{i,l}
            \\
        \nonumber
            & \st &
             \gamma_{i,l} = - M(y_i, l) \H_l^\T(\x_i)\w,
            \\
        \nonumber
            &&
             \w \succeq \0, \wone = \theta.
        \end{eqnarray}
        The Lagrangian of the above program is
        \begin{eqnarray}
        \nonumber
            L(\w, \gammav, \q, \u, r) =
                \sumNL \exp \gamma_{i,l}
                - \q^\T\w
                + r(\1^\T\w - \theta)
        \\
        \label{PF:MO-Lagr}
                - \sumNL u_{i,k} \left(\gamma_{i,k} + M(y_i, l) \H_l^\T(\x_i)\w \right)
        \end{eqnarray}
        with $\q \succeq \0$.
        The Lagrange dual function is defined as the infimum value of the Lagrangian
        over variables $\w$ and $\gammav$.
        \begin{eqnarray}
        \nonumber
          \inf_{\w, \gammav}\ L
            &=&
              \inf_{\w, \gammav}\
                \sumNL \exp{\gamma_{i,l}}
                - \sumNL u_{i,l}\gamma_{i,l}
                - r\theta
        \\
        \nonumber
            &&
                - \Big(
                    \overbrace{
                        \sumNL u_{i,l}M(y_i, l) \H_l^\T(\x_i)
                            + \q^\T
                            - r\1^\T
                            }^{ {\rm must\ be\ {\bf 0} } }
                  \Big)\w
        \\
        \label{EQ:MO-dsol-0}
            &=&
              \inf_{\gammav}\
                \sumNL \exp{\gamma_{i,l}}
                - \sumNL u_{i,l}\gamma_{i,l}
                - r\theta
        \\
        \nonumber
            &=&
                -\sumNL
                \overbrace{
                    \sup_{\gammav}(u_{i,l}\gamma_{i,l} - \exp{\gamma_{i,l}})
                    }^{  {\rm conjugate\ of\ exp. \ function } }
                - r\theta
        \\
        \label{PF:MO-dufn}
            &=&
                -\sumNL (u_{i,l} \log u_{i,l} - u_{i,l})
                - r\theta.
        \end{eqnarray}
        The convex conjugate (or Fenchel duality) of a function
        $f: \Rset \rightarrow \Rset$ is defined as
        \begin{eqnarray}
        \label{EQ:conjugate}
            g(y) = \sup_{x \in \dom f} (y^\T x - f(x)).
        \end{eqnarray}
        Here we have used the fact that the conjugate
        of the exponential function $f(\gamma) = e^\gamma$ is $g(u) = u\log u - u$,
        if and only if $u \geq 0$. $0\log 0$ is interpreted as $0$.

        For each pair $(\u, r)$, the dual function gives a lower bound
        on the optimal value of the primal problem \eqref{EQ:MO-prim}.
        Through maximizing the dual function, the best bound can be obtained.
        This is exactly the dual problem we derived.
        After eliminating $\q$ and collecting all the constraints,
        we complete the proof.
    \end{proof}

    \comment{
    %
    If we minimize the logarithmic version of the cost function in \eqref{EQ:MO-prim}, \ie,
    \[
        \log \left(     \sumNL \exp \left( - M(y_i, l) \H_l^\T(\x_i) \w \right) 
             \right) ,
    \]
    the optimization problem keeps unchanged because $ \log(\cdot) $ is
    monotonically increasing. 
    The corresponding dual can be written as 
    \begin{eqnarray}
    \label{EQ:MO-dual-min}
        & \min\limits_{r,\u}&
           r \theta
           + \sumNL u_{i,l}\log u_{i,l}
    \\
    \nonumber
        & \st &
           \sumNL u_{i,l} M(y_i, l) h_l^{(j)}(\x_i) \leq r,
           j = 1,\dots, T,
    \\
    \nonumber
        & &
           \u \succeq \0, \1^\T \u = 1.
    \end{eqnarray}
    This change  of $\Lone$ norm term is equivalent to replace the soft constraint
    by a hard constraint \cite{vapnik1998statistical}.
    }

    Since the primal \eqref{EQ:MO-prim} is a convex problem
    and satisfies Slater's condition \cite{boyd2004convex},
    the strong duality holds, which means maximizing the cost function in \eqref{EQ:MO-dual}
    over dual variables $\u$ and $r$ is equivalent to solve the problem \eqref{EQ:MO-prim}.
    Then if a dual optimal solution $(\u^*, r^*)$ is known,
    any primal optimal point is also a minimizer of $L(\w, \gammav, \q, \u^*, r^*)$.
    In other words, if the primal optimal solution $\w^*$ exists,
    it can be obtained by minimizing $L(\w, \gammav, \q, \u^*, r^*)$
    of the following function
    (see \eqref{EQ:MO-dsol-0}):
    \begin{equation}
    \label{EQ:MO-dsol-1}
    \begin{split}
    &   - r^*\theta
        + \sumNL
            \Big(
                \exp \big(- M(y_i, l) \H_l^\T(\x_i)\w \big)
                \\
    & \qquad \qquad \qquad \quad \quad
                - u^*_{i,l} \big(- M(y_i, l) \H_l^\T(\x_i)\w \big)
            \Big).
    \end{split}
    \end{equation}
    We can also use KKT conditions to establish the relationship
    between the primal variables and dual variables \cite{boyd2004convex}.

\subsection{\MultiA: Totally Corrective Boosting based on Column Generation}

    Clearly, we can not obtain the optimal solution of \eqref{EQ:MO-dual}
    until all the weak hypotheses in constraints become available.
    In order to solve this optimization problem,
    we use an optimization technique termed column generation
    \cite{demiriz2002linear}.
    The concept of column generation is adding one constraint at a time
    to the dual problem until an optimal solution is identified.
    In our case, we find the weak classifier at each iteration
    that most violates the constraint in the dual.
    For \eqref{EQ:MO-dual}, such a multidimensional weak classifier
    $\h ^* (\cdot) = [h_1 ^* (\cdot) \cdots h_L ^* (\cdot)]^\T$
    can be found by solving the following problem:
    \begin{equation}
    \label{EQ:MO-CGmc}
        \h ^* (\cdot) =
            \argmax_{ \h (\cdot) } \
                \sum_{l=1}^L \sum_{i=1}^N u_{i,l} M(y_i, l) h_l^{}(\x_i),
    \end{equation}
    which is equivalent to solve $L$ subproblems:
    \begin{equation}
    \label{EQ:MO-CGsc}
        h_l ^* (\cdot) =
            \argmax_{ h_l (\cdot) } \
            \sumN u_{i,l} M(y_i, l)h_l(\x_i), \
        k = 1,\dots, L.
    \end{equation}
    If we view $u_{i,l}$ as the weight of the coded training example
    $\left(\x_i, M(y_i, l)\right)$, $ i = 1,\dots, N$, $ l = 1,\dots, L$,
    this is exactly the same as the strategy AdaBoost.MO uses.
    That is, to find $L$ weak classifiers that produce the smallest
    weighted training error
    (maximum algebraic sum of weights of correctly classified data)
    with respect to the current weight distribution $\u$.

    When a new constraint is added into the dual program,
    the optimal value of this maximization problem \eqref{EQ:MO-dual} would decrease.
    Accordingly the primal optimal value decreases too because of the zero duality gap.
    The primal problem is convex,
    which assures that our optimization problem will converges to the global extremum.
    In practice, \MultiA{} converges quickly on our tested datasets.

    Next we need to find the connection between the primal variables $\w$
    and dual variables $\u$ and $r$.
    According to KKT conditions, since the primal optimal $\w^*$
    minimizes \eqref{EQ:MO-dsol-1} over $\w$,
    its gradient must equals to $0$ at $\w^*$.
    Thus we have
    \begin{equation}
    \label{EQ:MO-KKT}
        u^*_{i,l} = \exp \left(- M(y_i, l) \H_l^\T(\x_i)\w^* \right)
    \end{equation}
    In our experiments, we have used MOSEK \cite{mosek} optimization software,
    which is a primal-dual interior-point solver.
    Both the primal and dual solutions are available at convergence.

\subsection{Lagrange Dual of AdaBoost.ECC and \MultiB}
\label{sec:ecc-dual}

    The primal-dual method is so general,
    actually, arbitrary boosting algorithms based on convex loss functions
    can be integrated into this framework.
    Next we investigate another multiclass boosting algorithm: AdaBoost.ECC.

    Let us denote $\rho_{i,c}[h^{(t)}] = \left( M(y_i, t) - M(c, t) \right) h^{(t)}(\x_i)$.
    If we define the margin of example $(\x_i, y_i)$ on hypothesis $h^{(t)}(\cdot)$ as
    \begin{eqnarray}
    \rho_{i}[h^{(t)}]
    \label{EQ:ECC-marg-t}
        & \triangleq &
            \min_{\cbey} \{ \rho_{i,c}[h^{(t)}] \}
    \\
    \nonumber
        & = &
            M(y_i, t) h^{(t)}(\x_i)
            - \max_{\cbey} \{ M(c, t)h^{(t)}(\x_i) \},
    \end{eqnarray}
    the margin on assembled classifier $\f(\cdot)$ can be computed as
    \begin{eqnarray}
    \rho_{i}[\f]
    \nonumber
        & \triangleq &
            \min_{\cbey} \{ \rho_{i,c}[\f] \}
    \\
    \nonumber
        & = &
            M(y_i, :) \f(\x_i)
            - \max_{\cbey} \{ M(c, :)\f(\x_i) \}
    \\
    \label{EQ:ECC-marg}
        & = &
            \frac{1}{\sum_t \wt} \min_{\cbey} \{ \sumT \wt \rho_{i,c}[h^{(t)}] \}.
    \end{eqnarray}
    AdaBoost.ECC tries to maximize the minimum margin
    by optimizing the following loss function \cite{sun2007unifying}:
    \begin{eqnarray}
    L_{\rm ECC}
        \nonumber
            & = &
                \sumNCR \exp
                \left( -
                    \left( M(y_i, :) - M(c, :) \right)
                    \F(\x_i)
                \right)
        \\
        \label{EQ:ECC-loss}
            & = &
                \sumNCR \exp
                \left( - \sumT \wt \rho_{i,c}[h^{(t)}] \right).
    \end{eqnarray}
    Denote $\P_{i,c} =
    [\rho_{i,c}[h^{(1)}] {\ } \rho_{i,c}[h^{(2)}] \cdots \rho_{i,c}[h^{(T)}] ]^\T $.
    Obviously, $\P_{i,y_i} = \0^\T$ for any example $\x_i$.
    Therefore the problem we are interested in can be equivalently written as:
    \begin{eqnarray}
      \label{EQ:ECC-prim}
        &\min\limits_{\w} &
         \sumNC \exp \left( - \P_{i,c}^\T \w \right)
        \\
      \nonumber
        & \st &
        \w \succeq \0, \wone = \theta.
    \end{eqnarray}
    Like AdaBoost.MO,
    we have added an $\Lone$ norm constraint to remove the scale ambiguity.
    Clearly, this is also a convex problem in $\w$ and strong duality holds.

    \begin{theorem}
    \label{TM:Dual-ECC}
        The Lagrange dual problem of \eqref{EQ:ECC-prim} is
        \begin{eqnarray}
        \label{EQ:ECC-dual}
          & \max\limits_{r,\u}&
           -r\theta
           - \sumNC u_{i,c} \log u_{i,c}
           + \sumNC u_{i,c}
          \\
        \nonumber
          & \st &
           \sumNC u_{i,c} \P_{i,c}^\T \preceq r \1^\T,
           \u \succeq \0.
        \end{eqnarray}
    \end{theorem}

    The derivation is very similar
    with that in the proof of Theorem \ref{TM:Dual-MO}.
    Notice that the first constraint has no effect on variables $u_{i,y_i}$
    since $\P_{i,y_i} = \0^\T$, $\forall i = 1,\dots, N$,
    however, the problem is still bounded
    because $-u \log u + u \leq 1$ for all $u \geq 0$.
    Actually, we have $u_{i,y_i} \equiv 1$ in the process of optimization.

    \begin{algorithm}
    \caption{Totally Corrective Multiclass Boosting}
    \label{Alg:MBoostTC}
    \begin{algorithmic}
    \STATE
        \Input\ training data $ (\x_i, y_i)$, $ i =1,\dots, N$; termination\\
        \btab   threshold $ \varepsilon > 0$;
                regularization parameter $\theta > 0$;\\
        \btab   maximum training iterations $T$.
    \STATE
        (1) \Init\\
        \tab   $t = 0$; $\w = 0$; $r = 0$;\\
        \tab   $u_{i,k} = \frac{1}{NK}$, $ i = 1,\dots, N $, $k = 1,\dots, K$. \\
    \WHILE{true}
    \STATE
        (2) Find a new weak classifier $\h^*(\cdot)$ by solving\\
        \tab    subproblem in column generation:\\
        \tab    $\h^* = \argmax_{h} \sum_i \sum_k u_{i,k} \rho_{i,k}[\h]$;
    \STATE
        (3) Check if dual problem is bounded by new constraint:\\
        \tab    \If $\sum_i \sum_k u_{i,k} \rho_{i,k}[\h^*] < r + \epsilon$, \then break;
    \STATE
        (4) Add new constraint to dual problem;
    \STATE
        (5) Solve dual problem to obtain updated $r$ and $\u$:\\
        \tab $\max_{r,\u} \hspace{1mm} -r\theta - \sum_i\sum_k u_{i,k} \log u_{i,k}+\1^\T \u$ \\
        \tab \hspace{1.2mm}  $\st \hspace{5mm} \sum_i\sum_k u_{i,k} \rho_{i,k}[\h^{(j)}] \leq r$,
                             $j = 1,\dots, t$; \\
        \tab \hspace{11mm}  $\u \succeq \0$;
    \STATE
        (6) $t = t + 1$; \If $t > T$, \then break;
    \ENDWHILE
    \STATE
        (7) Calculate the primal variable $\w$ according to dual\\
        \tab    solutions and KKT condition.
    \STATE
        \outpt\ $\F(\cdot) = \sum_{j=1}^t \wj \h^{(j)}(\cdot)$.
    \end{algorithmic}
    \end{algorithm}

    To solve this dual problem,
    we also employ the idea of column generation.
    Hence at $t\Th$ iteration,
    such an optimal weak classifier can be found by
    \begin{equation}
    \label{EQ:ECC-CG1}
        h ^* (\cdot) =
            \argmax_{ h (\cdot) } \
                \sum_{i=1}^N
                    \left( \sum_c
                        u_{i,c} \left( M(y_i, t) - M(c, t) \right)
                    \right) h(\x_i).
    \end{equation}
    Notice that $M(y_i, t)-M(c, t) = 2 M(y_i, t) \ind ( M(y_i, t)$ $\neq M(c,t) )$,
    so if we rewrite \eqref{EQ:ECC-CG1} in the following form:
    \begin{equation}
    \begin{split}
        h ^* (\cdot)
        & =
            \argmax_{ h (\cdot) } \
                \sum_{i = 1}^N
                \bigg(
                    \Big(
                        \sum_c
                        u_{i,c}\ \ind \big( M(y_i, t) \neq M(c, t) \big)
                    \Big) \\
        & \qquad \qquad \qquad \qquad M(y_i, t) h(\x_i)
                \bigg),
    \label{EQ:ECC-CG2}
    \end{split}
    \end{equation}
    it it straightforward to show that we can use the same strategy with AdaBoost.ECC
    to obtain weak classifiers.
    To be more precise,
    the strategy is to minimize the training error with respect to the 
    mislabel distribution.

    Looking at optimization problem \eqref{EQ:MO-dual} and \eqref{EQ:ECC-dual},
    they are quite similar.
    If we denote $\rho_{i,l}[\h^{(t)}] = M(y_i, l)h_l^{(t)}(\x_i)$
    in the first case, the margin of example $(\x_i, y_i)$ on hypothesis $\h^{(t)}$ would be
    $\rho_i[\h^{(t)}] = \min_l \{ \rho_{i,l}[\h^{(t)}] \}$,
    and also
    \begin{eqnarray}
    \rho_i[\f]
    \nonumber
        & = &
            \min_l \{ M(y_i, l)f_l(\x_i) \}
    \\
    \label{EQ:MO-marg}
        & = &
            \frac{1}{\wone} \min_l \{ \sum_t \wt \rho_{i,l}[\h^{(t)}] \}.
    \end{eqnarray}
    Based on different definitions of margin,
    these two problems share exactly the same expression.
    To summarize, we combine the algorithms that we proposed in this section
    and give a general framework for multiclass boosting in Algorithm \ref{Alg:MBoostTC}.

{

\subsection{Hinge Loss Based Multiclass Boosting}
\label{sec:hing}

    Within this framework, we can devise other boosting algorithms
    based on different loss functions.
    Here we give an example.
    According to  \eqref{EQ:MO-predict}, if a pattern $\x_i$ is well classified,
    it should satisfy
    \begin{equation}
    \label{EQ:hin-1}
        M(y_i, :) \F(\x) \geq M(c, :) \F(\x) + 1, \forall c \neq y_i
    \end{equation}
    with $\varphi > 0$. Define the hinge loss for $\x_i$ as
    \begin{equation}
    \label{EQ:hin-loss}
        \xi_i = \max_c \{ M(c, :) \F(\x) + 1 - \dtci \} - M(y_i, :) \F(\x),
    \end{equation}
    where $\dtci = 1$ if $c = y_i$, else $0$.
     That is to say, if $\x_i$ is fully separable,
    then $\xi_i = 0$;
    else it suffers a loss $\xi_i > 0$.
    So the problem we are interested in is to find a classifier
    $\F(\cdot) = [\h^{(1)}(\cdot) \cdots \h^{(T)}(\cdot)]^\T \w$
    through the following optimization:
    \begin{eqnarray}
        \label{EQ:hin-prim}
        &\min \limits_{\xis, \w} &
            \sumN \xi_i
        \\
        \nonumber
        & \st &
            M(c, :) \F(\x) + 1 - \dtci - M(y_i, :) \F(\x) \leq \xi_i, \forall i, c;
        \\
        \nonumber
        &&
            \w \succeq 0;
            \wone = \theta.
    \end{eqnarray}

    \begin{theorem}
    \label{TM:Dual-hinge}
        The equivalent dual problem of \eqref{EQ:hin-prim} is
        \begin{eqnarray}
            \label{EQ:hin-dual}
            &\min \limits_{r, \u} &
                r\theta + \sumNC \dtci u_{i,c}
            \\
            \nonumber
            & \st &
                \sumNC \u_{i,c}
                \left( M(y_i, :) - M(c, :) \right) \h^{(j)}(\x_i) \leq r, \forall j;
            \\
            \nonumber
            &&
                \u \succeq 0;
                \sum_{c=1}^C u_{i,c} = 1, \forall i.
        \end{eqnarray}
    \end{theorem}
    \begin{proof}
        The Lagrangian of this program is a linear function on both $\xis$ and $\w$,
        therefore, the proof is easily done by
        letting the partial derivations on them equal zero,
        and substituting the results back.
    \end{proof}

    Suppose the length of a codeword $M(c, :)$ is $L$.
    Using the idea of column generation,
    we can iteratively obtain weak hypotheses and the associated coefficients
    by solving
    \begin{equation}
    \label{EQ:hin-cg1}
        \h ^* (\cdot) =
            \argmax_{ \h (\cdot) } \
                \sum_{i,c} u_{i,c}
                    \sum_{l = 1}^L \left( M(y_i, t) - M(c, t) \right) \h(\x_i),
    \end{equation}
    or $L$ subproblems
    \begin{equation}
    \label{EQ:hin-cg2}
        h_l ^* (\cdot) =
            \argmax_{ h (\cdot) } \
                \sum_{i=1}^N
                    \left( \sum_c
                        u_{i,c} \left( M(y_i, l) - M(c, l) \right)
                    \right) h(\x_i).
    \end{equation}
    This is exactly the same as  \eqref{EQ:ECC-CG1}. In other words,
    we can follow the same procedures as in AdaBoost.ECC to obtain each entry of weak hypotheses.
}

    The proposed boosting framework may inspire us 
    to design other multiclass boosting algorithms in the primal by considering different 
    coding strategies and loss functions.
    We can use a predefined coding matrix to train a set of multidimensional hypotheses
    as in AdaBoost.MO, at the same time, penalize the mismatched labels as in AdaBoost.ECC.
    It seems to be a mixture of these two algorithms.
    However, this is beyond the scope of this paper.

\subsection{Totally Corrective Update}
\label{sec:tc}

    \begin{figure}[ht!]
    \begin{center}
        \subfigure[]
        {
            \includegraphics[width=0.35\textwidth,clip]{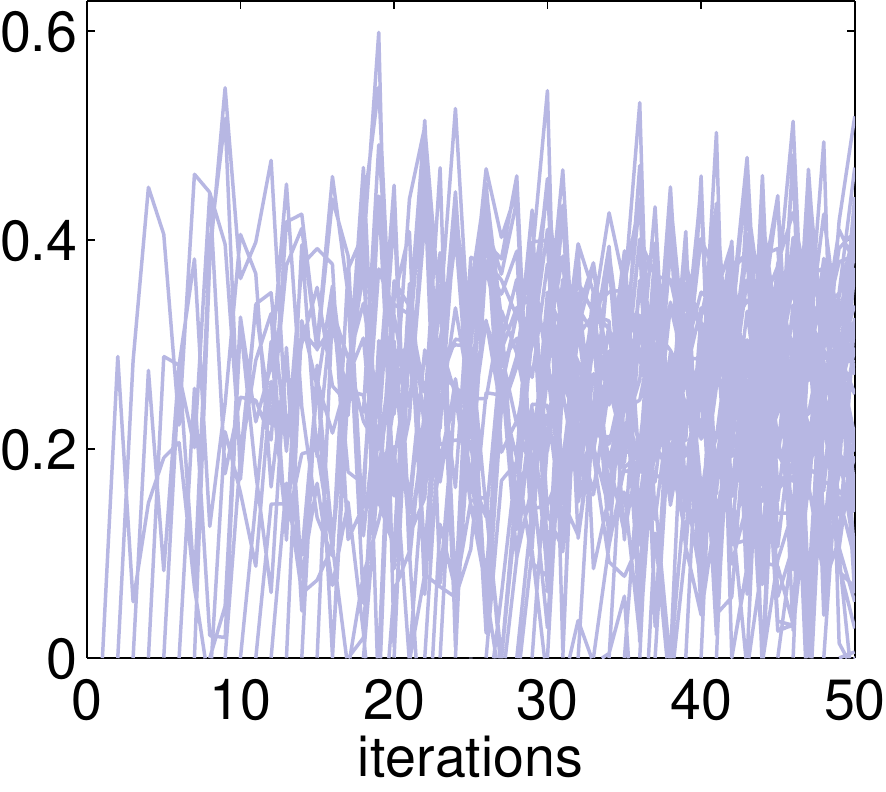}
            \label{fig:abr}
        }
        \subfigure[]
        {
            \includegraphics[width=0.35\textwidth,clip]{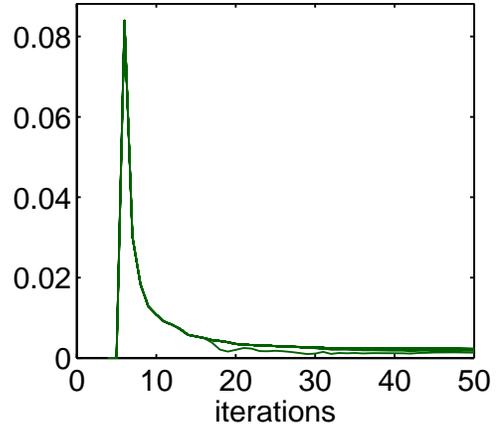}
            \label{fig:cgr}
        }
    \end{center}
    \caption
    {
        Inner products between the new distribution and the past mistake vectors
        $\Fset(t) = \{ {\u^{(t+1)}}^\T \rhos[h^{(j)}] \}_{j = 1}^{50}$, $t \geq j$
        on the training data of \win.
        In (a) AdaBoost.MO, $\Fset(t)$ starts with $\Fset(j) = 0$
        (corrective update) and quickly becomes uncontrollable;
        in (b) \MultiA, $\Fset(t)$ is consistently bounded by $r$.
	}
    \label{fig:r}
    \end{figure}

    Next, we provide an alternative explanation for the new boosting methods.
    In Algorithm \ref{Alg:MBoostTC}, suppose $t$ weak hypotheses have been found
    while the cost function still does not converge
    (\ie, the $t\Th$ constraint is not satisfied).
    To obtain the updated weight distribution $\u$ for the next iteration,
    we need to solve the following optimization problem:
    \begin{eqnarray}
    \label{EQ:TC-dual}
        & \min\limits_{r,\u}&
           r \theta + \sumNK u_{i,k}\log u_{i,k}
    \\
    \label{EQ:TC-r}
        & \st &
           \u^\T \rhos[h^{(j)}] \leq r,
           j = 1,\dots, t,
    \\
    \nonumber
        & &
           \u \succeq \0, \1^\T \u = 1,
    \end{eqnarray}
    where $\u^\T \rhos[h^{(j)}] = \sumNK u_{i,k} \rho_{i,k}[h^{(j)}]$.
    In other words, ${\u^{(t+1)}}^\T \rhos[h^{(j)}] \leq r$ holds for $j = 1,\dots, t$.

    In \cite{kivinen1999boosting}, Kivinen and Warmuth have shown that
    the \emph{corrective} update of weight distribution in standard boosting algorithms
    can be viewed as a solution to the divergence minimization problem
    with the constraint of ${\u^{(t+1)}}^\T \rhos[h^{(t)}] = 0$.
    This means that the new distribution $\u^{(t+1)}$ should be closest with $\u^{(t)}$
    but uncorrelated with the \emph{mistakes} made by the current weak hypothesis.
    If $\u^{(t+1)}$ is further required to be orthogonal to all the past $t$ mistake vectors:
    \begin{equation}
    \label{EQ:TC-0}
        {\u^{(t+1)}}^\T \rhos[h^{(j)}] = 0, \forall j = 1,\dots, t,
    \end{equation}
    the update technique is called a \emph{totally corrective} update.
    At the same time, the previous $t$ hypotheses
    receive an increment in their coefficient respectively.

    Notice that there is not always an exact solution to \eqref{EQ:TC-0}.
    Even if a solution exists, this optimization problem
    might become too complex as $t$ increases.
    Some attempts have been made to obtain an approximate result.
    Kivinen and Warmuth \cite{kivinen1999boosting} suggested using an iterative approach,
    which is actually a column generation based optimization procedure.
    Oza \cite{oza2003boosting} proposed to construct $\u^{(t+1)}$
    by averaging $t+1$ distributions computed from standard AdaBoost update,
    which is in fact the least-squares solution to linear equations
    $[\rhos[h^{(1)}] \cdots \rhos[h^{(t)}]]^\T\u = \0$.
    The stopping criteria of his AveBoost implied a continuous descent
    of inner products ${\u^{(t+1)}}^\T \rhos[h^{(j)}]$
    as in our algorithm, but exhibited in a heuristic way.
    Jiang and Ding \cite{jiang2010partially} also noticed the feasibility problem
    of \eqref{EQ:TC-0}. They proposed to solve a subset, say $m$ equations instead of the entire,
    however, they did not indicate how to choose $m$.

    In our algorithms, all the coefficients $\{ \wj \}_{j=1}^t$
    associated to weak hypotheses are updated at each iteration,
    while the inner products of distribution $\u^{(t+1)}$
    and $\rhos[h^{(j)}]$ are consistently bounded by $r$.
    In this sense, our multiclass boosting learning can be considered as
    a relaxed version of totally corrective algorithm
    with slack variable $r$.
    Figure \ref{fig:r} illustrates the difference between corrective algorithm and our algorithm.
    Intuitively, it is more feasible to solve inequalities \eqref{EQ:TC-r}
    than the same size of equations \eqref{EQ:TC-0},
    although most (but not all, see Fig. \ref{fig:r}) equalities
    in \eqref{EQ:TC-r} are satisfied during the optimization process.
    Analogous relaxation methods have appeared in LPBoost \cite{demiriz2002linear}
    and TotalBoost \cite{warmuth2006totally}.
    In contrast, our algorithm simultaneously restrains
    the distribution divergence and correlations of hypotheses.
    The parameter $\theta$ is a trade-off that controls the balance.
    To our knowledge, this is the first algorithm
    to introduce the totally corrective concept to multiclass boosting learning.

    If we remove the constraints on $\w$ in primal,
    for example, constraints \eqref{EQ:MO-prim-cons},
    the Lagrange dual problem turns to be a totally corrective minimizer of negative entropy:
    \begin{eqnarray}
    \label{EQ:TC-div-min}
        & \min\limits_{r,\u}&
           \sumNK u_{i,k}\log u_{i,k}
    \\
    \nonumber
        & \st &
           \u^\T \rhos[h^{(j)}] = 0,
           j = 1,\dots, t,
    \\
    \nonumber
        & &
           \u \succeq \0, \1^\T \u = 1,
    \end{eqnarray}
    which is similar to the explanation of corrective update in \cite{kivinen1999boosting},
    where the distribution divergence
    is measured by relative entropy (Kullback-Leibler divergence).
    However, the constraints on $\w$ are quite important as we discussed before,
    and should not be simply removed.
    Therefore, it seems reasonable to use a relaxed version,
    instead of standard totally corrective constraints
    for distribution update of boosting learning.

    It has been proven that totally corrective boosting reduces the upper bound
    on training error more aggressively than standard corrective boosting
    (apparently including AdaBoost.MO and AdaBoost.ECC)
    \cite{sochman2004adaboost,jiang2010partially},
    which performs a slowly stage-wise gradient descent procedure on the loss function.
    Thus, our boosting algorithms can be expected
    to be faster in convergence than their counterparts.
    The fast convergence speed is advantageous in reducing the training cost
    and producing a classifier composed of fewer weak hypotheses \cite{shen2010dual,sochman2004adaboost}.
    Further, a simplification in strong classifier leads to a speedup of classification,
    which is critical to many applications, especially those with real-time requirements.

    \setlength{\tabcolsep}{8pt}
    \begin{table}[!t]
    \centering
    \caption{Multiclass UCI datasets}
    \label{table:datasets}
    \begin{tabular}{l|r|r|r|r}
    \hline
        dataset & \#train & \#test & \#attribute & \#class \\
    \hline
        svmguide2 & 391   & -    & 20 & 3  \\
        svmguide4 & 300   & 312  & 10 & 6  \\
        wine      & 178   & -    & 13 & 3  \\
        iris      & 150   & -    & 4  & 3  \\
        glass     & 214   & -    & 9  & 6  \\
        thyroid   & 3772  & 3428 & 21 & 3  \\
        dna       & 2000  & 1186 & 180 & 3 \\
        vehicle   & 846   & -    & 18 & 4  \\
    \hline
    \end{tabular}
    \end{table}
    \setlength{\tabcolsep}{1.4pt}

    \setlength{\tabcolsep}{8pt} 
    \begin{table*}[!ht]
    \centering
    \caption{
        Training and test errors(including mean and standard deviation)
        of AdaBoost.MO, \MultiA, AdaBoost.ECC and \MultiB{} on UCI datasets.
        Average results of 20 repeated tests are reported.
        Base learners are decision stumps.
        Results in bold are better than their counterparts.
    }
    \label{table:dstump}
    \begin{tabular}{llllllll}
    \hline
    dataset & algorithm
    		& train error 50	& train error 100	& train error 500
    		& test error 50	& test error 100	& test error 500	\\
    \hline
    	svmguide2
    	    &AB.MO	
    			& 0.016$\pm$0.005 	
    			& 0.000$\pm$0.000 	
    			& 0$\pm$0		
    			& 0.225$\pm$0.031 	
    			& \textbf{0.212$\pm$0.030 	}
    			& \textbf{0.229$\pm$0.024 	} \\
    	    &TC.MO	
    			& \textbf{0.011$\pm$0.003 	}
    			& \textbf{0$\pm$0		}
    			& 0$\pm$0		
    			& \textbf{0.224$\pm$0.024 	}
    			& 0.221$\pm$0.038 	
    			& 0.233$\pm$0.028 	 \\
    	\cline{2-8}
    	    &AB.ECC	
    			& 0.049$\pm$0.009 	
    			& 0.003$\pm$0.003 	
    			& 0$\pm$0		
    			& 0.253$\pm$0.032 	
    			& 0.226$\pm$0.031 	
    			& 0.226$\pm$0.030 	 \\
    	    &TC.ECC	
    			& \textbf{0.032$\pm$0.009 	}
    			& \textbf{0.001$\pm$0.002 	}
    			& 0$\pm$0		
    			& \textbf{0.242$\pm$0.031 	}
    			& \textbf{0.223$\pm$0.034 	}
    			& \textbf{0.221$\pm$0.028 	} \\
    \hline
    	svmguide4
    	    &AB.MO	
    			& 0.041$\pm$0.004 	
    			& 0.016$\pm$0.002 	
    			& 0$\pm$0		
    			& 0.204$\pm$0.025 	
    			& \textbf{0.194$\pm$0.017 	}
    			& \textbf{0.190$\pm$0.017 	} \\
    	    &TC.MO	
    			& \textbf{0.039$\pm$0.004 	}
    			& \textbf{0.014$\pm$0.002 	}
    			& 0$\pm$0		
    			& \textbf{0.203$\pm$0.023 	}
    			& 0.196$\pm$0.016 	
    			& 0.193$\pm$0.018 	 \\
    	\cline{2-8}
    	    &AB.ECC	
    			& 0.180$\pm$0.021 	
    			& 0.115$\pm$0.023 	
    			& 0$\pm$0		
    			& 0.285$\pm$0.023 	
    			& 0.262$\pm$0.025 	
    			& \textbf{0.233$\pm$0.023 	} \\
    	    &TC.ECC	
    			& \textbf{0.158$\pm$0.023 	}
    			& \textbf{0.087$\pm$0.020 	}
    			& 0$\pm$0		
    			& \textbf{0.275$\pm$0.022 	}
    			& \textbf{0.245$\pm$0.028 	}
    			& 0.237$\pm$0.022 	 \\
    \hline
    	wine
    	    &AB.MO	
    			& 0$\pm$0		
    			& 0$\pm$0		
    			& 0$\pm$0		
    			& 0.032$\pm$0.015 	
    			& \textbf{0.030$\pm$0.029 	}
    			& \textbf{0.031$\pm$0.028 	} \\
    	    &TC.MO	
    			& 0$\pm$0		
    			& 0$\pm$0		
    			& 0$\pm$0		
    			& 0.032$\pm$0.016 	
    			& 0.031$\pm$0.023 	
    			& 0.032$\pm$0.026 	 \\
    	\cline{2-8}
    	    &AB.ECC	
    			& 0$\pm$0		
    			& 0$\pm$0		
    			& 0$\pm$0		
    			& 0.026$\pm$0.020 	
    			& 0.032$\pm$0.015 	
    			& 0.026$\pm$0.024 	 \\
    	    &TC.ECC	
    			& 0$\pm$0		
    			& 0$\pm$0		
    			& 0$\pm$0		
    			& 0.026$\pm$0.022 	
    			& 0.032$\pm$0.034 	
    			& 0.026$\pm$0.019 	 \\
    \hline
    	iris
    	    &AB.MO	
    			& 0.000$\pm$0.001 	
    			& 0$\pm$0		
    			& 0$\pm$0		
    			& 0.062$\pm$0.026 	
    			& \textbf{0.064$\pm$0.025 	}
    			& 0.060$\pm$0.032 	 \\
    	    &TC.MO	
    			& \textbf{0$\pm$0		}
    			& 0$\pm$0		
    			& 0$\pm$0		
    			& \textbf{0.062$\pm$0.026 	}
    			& 0.067$\pm$0.023 	
    			& \textbf{0.057$\pm$0.025 	} \\
    	\cline{2-8}
    	    &AB.ECC	
    			& 0$\pm$0		
    			& 0$\pm$0		
    			& 0$\pm$0		
    			& \textbf{0.057$\pm$0.024 	}
    			& \textbf{0.062$\pm$0.026 	}
    			& 0.053$\pm$0.032 	 \\
    	    &TC.ECC	
    			& 0$\pm$0		
    			& 0$\pm$0		
    			& 0$\pm$0		
    			& 0.061$\pm$0.021 	
    			& 0.067$\pm$0.020 	
    			& \textbf{0.051$\pm$0.028 	} \\
    \hline
    	glass
    	    &AB.MO	
    			& 0.026$\pm$0.003 	
    			& 0.003$\pm$0.001 	
    			& 0$\pm$0		
    			& \textbf{0.275$\pm$0.034 	}
    			& \textbf{0.246$\pm$0.061 	}
    			& \textbf{0.268$\pm$0.051 	} \\
    	    &TC.MO	
    			& \textbf{0.022$\pm$0.003 	}
    			& \textbf{0.002$\pm$0.001 	}
    			& 0$\pm$0		
    			& 0.280$\pm$0.039 	
    			& 0.252$\pm$0.059 	
    			& 0.273$\pm$0.047 	 \\
    	\cline{2-8}
    	    &AB.ECC	
    			& 0.168$\pm$0.032 	
    			& 0.078$\pm$0.018 	
    			& 0$\pm$0		
    			& 0.352$\pm$0.052 	
    			& 0.313$\pm$0.043 	
    			& \textbf{0.298$\pm$0.044 	} \\
    	    &TC.ECC	
    			& \textbf{0.113$\pm$0.030 	}
    			& \textbf{0.020$\pm$0.013 	}
    			& 0$\pm$0		
    			& \textbf{0.327$\pm$0.048 	}
    			& \textbf{0.302$\pm$0.035 	}
    			& 0.306$\pm$0.045 	 \\
    \specialrule{1pt}{0pt}{0pt}
    	thyroid
    	    &AB.MO	
    			& 0.003$\pm$0.001 	
    			& 0.001$\pm$0.000 	
    			& 0$\pm$0		
    			& 0.006$\pm$0.001 	
    			& \textbf{0.006$\pm$0.001 	}
    			& \textbf{0.006$\pm$0.002 	} \\
    	    &TC.MO	
    			& \textbf{0.001$\pm$0.001 	}
    			& \textbf{0.001$\pm$0.001 	}
    			& 0$\pm$0		
    			& \textbf{0.006$\pm$0.001 	}
    			& 0.006$\pm$0.001 	
    			& 0.006$\pm$0.001 	 \\
    	\cline{2-8}
    	    &AB.ECC	
    			& 0.006$\pm$0.001 	
    			& 0.002$\pm$0.001 	
    			& 0$\pm$0		
    			& 0.010$\pm$0.002 	
    			& 0.008$\pm$0.002 	
    			& 0.005$\pm$0.001 	 \\
    	    &TC.ECC	
    			& \textbf{0.000$\pm$0.001 	}
    			& \textbf{0$\pm$0		}
    			& 0$\pm$0		
    			& \textbf{0.006$\pm$0.002 	}
    			& \textbf{0.006$\pm$0.002 	}
    			& \textbf{0.004$\pm$0.000 	} \\
    \hline
    	dna
    	    &AB.MO	
    			& 0.053$\pm$0.002 	
    			& 0.040$\pm$0.002 	
    			& \textbf{0.028$\pm$0.002 	}
    			& \textbf{0.076$\pm$0.007 	}
    			& 0.066$\pm$0.006 	
    			& 0.061$\pm$0.006 	 \\
    	    &TC.MO	
    			& \textbf{0.052$\pm$0.004 	}
    			& \textbf{0.039$\pm$0.003 	}
    			& 0.029$\pm$0.002 	
    			& 0.078$\pm$0.007 	
    			& \textbf{0.064$\pm$0.006 	}
    			& \textbf{0.054$\pm$0.005 	} \\
    	\cline{2-8}
    	    &AB.ECC	
    			& 0.070$\pm$0.004 	
    			& 0.049$\pm$0.005 	
    			& \textbf{0.017$\pm$0.004 	}
    			& 0.089$\pm$0.008 	
    			& 0.077$\pm$0.009 	
    			& 0.069$\pm$0.005 	 \\
    	    &TC.ECC	
    			& \textbf{0.059$\pm$0.005 	}
    			& \textbf{0.041$\pm$0.006 	}
    			& 0.028$\pm$0.005 	
    			& \textbf{0.083$\pm$0.008 	}
    			& \textbf{0.070$\pm$0.006 	}
    			& \textbf{0.065$\pm$0.004 	} \\
    \hline
    	vehicle
    	    &AB.MO	
    			& 0.099$\pm$0.003 	
    			& 0.073$\pm$0.003 	
    			& \textbf{0.018$\pm$0.000 	}
    			& 0.249$\pm$0.020 	
    			& 0.245$\pm$0.019 	
    			& 0.212$\pm$0.010 	 \\
    	    &TC.MO	
    			& \textbf{0.086$\pm$0.009 	}
    			& \textbf{0.048$\pm$0.007 	}
    			& 0.019$\pm$0.027 	
    			& \textbf{0.241$\pm$0.020 	}
    			& \textbf{0.231$\pm$0.018 	}
    			& \textbf{0.211$\pm$0.021 	} \\
    	\cline{2-8}
    	    &AB.ECC	
    			& 0.271$\pm$0.010 	
    			& 0.207$\pm$0.011 	
    			& 0.096$\pm$0.010 	
    			& 0.359$\pm$0.022 	
    			& 0.300$\pm$0.021 	
    			& \textbf{0.249$\pm$0.017 	} \\
    	    &TC.ECC	
    			& \textbf{0.208$\pm$0.018 	}
    			& \textbf{0.140$\pm$0.019 	}
    			& \textbf{0.055$\pm$0.011 	}
    			& \textbf{0.327$\pm$0.022 	}
    			& \textbf{0.287$\pm$0.024 	}
    			& 0.257$\pm$0.028 	 \\
    \hline
    \end{tabular}
    \end{table*}
    \setlength{\tabcolsep}{1.4pt}

\section{Experiments} \label{Sect:Exps}

    In this section, we perform several experiments
    to compare our totally corrective multiclass boosting algorithms with previous work,
    including \MultiA{} against its stage-wise counterpart AdaBoost.MO,
    and \MultiB{} against its stage-wise counterpart AdaBoost.ECC.
    For \MultiA{} and AdaBoost.MO, we design error-correcting outputs codes
    (ECOC) \cite{dietterich1995solving} as the coding matrix.
    For our new algorithms, we solve the dual optimization problems
    using the off-the-shelf MOSEK package \cite{mosek}.

    The datasets used in our experiments are collected from UCI repository \cite{blake-uci}.
    A summary is listed in Table \ref{table:datasets}.
    We preprocess these datasets as follows:
    if it is provided with a pre-specified test set,
    the partitioning setup is retained,
    otherwise $70\%$ samples are used for training and the other $30\%$ for test.
    On each test, these two sets are merged and rebuilt by random selecting examples.
    To keep the balance of multiclass problems,
    examples associated with the same class are carefully split in proportion.
    The boosting algorithms are conducted on new sets.
    This procedure is repeated $20$ times.
    We report the average value as the experimental result.

    In the first experiment, we choose decision stumps as the weak learners.
    As a binary classifier, decision stump is extensively used due to its simplicity.
    The parameters are preset as follows.
    The maximum number of training iterations is set to $50$, $100$ and $500$.
    %
    %
    An important parameter to be tuned is the regularization parameter $\theta$,
    which equals to the $\Lone$ norm of coefficient vector associated with weak hypotheses.
    A simple method to choose $\theta$ is running the corresponding stage-wise algorithms
    on the same data and then computing the algebraic sum: $\theta = \sum_j \wj$.
    For datasets \svt, \svf, \win, \iri{} and \gla,
    which contain a small number of examples, we use this method. 
    The same strategy has been used in \cite{shen2010dual} to test {\em binary} totally corrective
    boosting.  
    
    For the others, we choose $\theta$
    from $\{ 2,5,8,$ $10,12,15,20,30,$ $ 40,45,$ $60,80,$ $100,120,$ $150,200 \}$
    by running a five-fold cross validation on training data.
    In particular, we use a pseudo-random code generator in the cross validations
    of AdaBoost.ECC and \MultiB, to make sure each candidate parameter is tested
    under the same coding strategy.

    The experimental results are reported in Table \ref{table:dstump}.
    As we can see, almost all the training errors of totally corrective algorithms are lower
    than their counterparts, except in the case both algorithms have converged to $0$.
    In Fig. \ref{fig:train}, we show the training error curves of some datasets
    when the training iteration number is set to $500$.
    Obviously, the convergence speed of our totally corrective boosting
    is much faster than the stage-wise one.
    This conclusion is consistent with the discussion in Section \ref{sec:tc}.
    Especially on \svt, \iri{} and \gla,
    new algorithms are around $50$ iterations faster than their counterparts.

    In terms of test error, it is not apparent which algorithm is better.
    Empirically speaking, the totally corrective boosting
    has a comparable generalization capability with the stage-wise version.
    It is noticeable that on \thy, \dna{} and \veh,
    where the regularization parameter is adjusted via cross-validation,
    our algorithms clearly outperform their counterparts.
    We conjecture that if we tune this parameter more carefully,
    the performance of new algorithms could be further improved.

    In the second experiment, we change the base learner with another binary classifier:
    Fisher's linear discriminant function (LDA).
    For simplicity, we only run AdaBoost.ECC and \MultiB{} at this time.
    All the parameters and settings are the same as in the first experiment.
    The results are reported in Table \ref{table:lda}.
    Again, the convergence speed of totally corrective boosting is faster
    than gradient descent version.
    We also notice that two algorithms of ECCs are better with LDAs
    than with decision stumps on \svt,
    but worse on \svf, \gla, \thy{} and \veh,
    although LDA is evidently stronger than decision stump.

    \setlength{\tabcolsep}{8pt} 
    \begin{table*}[!ht]
    \centering
    \caption{
        Training and test errors (including mean and standard deviation)
        of AdaBoost.ECC and \MultiB.
        The average results of 20 repeated tests are reported.
        Base learners are Fisher's linear discriminant functions.
    }
    \label{table:lda}
    \begin{tabular}{llllllll}
    \hline
    dataset & algorithm
    		& train error 50	& train error 100	& train error 500
    		& test error 50	& test error 100	& test error 500	\\
    \hline
    	svmguide2
    	    &AB.ECC	
    			& 0.055$\pm$0.021 	
    			& 0.003$\pm$0.006 	
    			& 0$\pm$0		
    			& \textbf{0.214$\pm$0.033 	}
    			& \textbf{0.224$\pm$0.025 	}
    			& \textbf{0.197$\pm$0.021 	} \\
    	    &TC.ECC	
    			& \textbf{0.031$\pm$0.013 	}
    			& \textbf{0$\pm$0		}
    			& 0$\pm$0		
    			& 0.228$\pm$0.044 	
    			& 0.229$\pm$0.024 	
    			& 0.221$\pm$0.027 	 \\
    \hline
    	svmguide4
    	    &AB.ECC	
    			& 0.323$\pm$0.038 	
    			& 0.208$\pm$0.045 	
    			& 0.000$\pm$0.001 	
    			& 0.457$\pm$0.036 	
    			& 0.428$\pm$0.049 	
    			& \textbf{0.328$\pm$0.037 	} \\
    	    &TC.ECC	
    			& \textbf{0.275$\pm$0.035 	}
    			& \textbf{0.149$\pm$0.033 	}
    			& \textbf{0$\pm$0		}
    			& \textbf{0.418$\pm$0.041 	}
    			& \textbf{0.377$\pm$0.051 	}
    			& 0.338$\pm$0.048 	 \\
    \hline
    	wine
    	    &AB.ECC	
    			& 0$\pm$0		
    			& 0$\pm$0		
    			& 0$\pm$0		
    			& 0.038$\pm$0.025 	
    			& 0.027$\pm$0.025 	
    			& 0.030$\pm$0.021 	 \\
    	    &TC.ECC	
    			& 0$\pm$0		
    			& 0$\pm$0		
    			& 0$\pm$0		
    			& \textbf{0.025$\pm$0.027 	}
    			& \textbf{0.020$\pm$0.022 	}
    			& \textbf{0.015$\pm$0.014 	} \\
    \hline
    	iris
    	    &AB.ECC	
    			& 0$\pm$0		
    			& 0$\pm$0		
    			& 0$\pm$0		
    			& \textbf{0.041$\pm$0.030 	}
    			& \textbf{0.040$\pm$0.024 	}
    			& 0.056$\pm$0.038 	 \\
    	    &TC.ECC	
    			& 0$\pm$0		
    			& 0$\pm$0		
    			& 0$\pm$0		
    			& 0.046$\pm$0.037 	
    			& 0.042$\pm$0.026 	
    			& \textbf{0.049$\pm$0.038 	} \\
    \hline
    	glass
    	    &AB.ECC	
    			& 0.194$\pm$0.033 	
    			& 0.080$\pm$0.021 	
    			& 0$\pm$0		
    			& 0.384$\pm$0.051 	
    			& \textbf{0.357$\pm$0.051 	}
    			& 0.369$\pm$0.052 	 \\
    	    &TC.ECC	
    			& \textbf{0.077$\pm$0.028 	}
    			& \textbf{0.001$\pm$0.002 	}
    			& 0$\pm$0		
    			& \textbf{0.374$\pm$0.046 	}
    			& 0.364$\pm$0.047 	
    			& \textbf{0.359$\pm$0.058 	} \\
    \specialrule{1pt}{0pt}{0pt} 
    	thyroid
    	    &AB.ECC	
    			& 0.041$\pm$0.004 	
    			& 0.035$\pm$0.005 	
    			& 0.001$\pm$0.001 	
    			& 0.048$\pm$0.006 	
    			& 0.046$\pm$0.004 	
    			& 0.032$\pm$0.004 	 \\
    	    &TC.ECC	
    			& \textbf{0.033$\pm$0.006 	}
    			& \textbf{0.018$\pm$0.010 	}
    			& \textbf{0$\pm$0		}
    			& \textbf{0.043$\pm$0.007 	}
    			& \textbf{0.040$\pm$0.004 	}
    			& \textbf{0.030$\pm$0.001 	} \\
    \hline
    	dna
    	    &AB.ECC	
    			& 0.000$\pm$0.000 	
    			& 0.000$\pm$0.000 	
    			& 0.000$\pm$0.000 	
    			& 0.081$\pm$0.007 	
    			& 0.084$\pm$0.009 	
    			& 0.077$\pm$0.010 	 \\
    	    &TC.ECC	
    			& 0.000$\pm$0.000 	
    			& 0.000$\pm$0.000 	
    			& \textbf{0.000$\pm$0.000 	}
    			& \textbf{0.068$\pm$0.008 	}
    			& \textbf{0.065$\pm$0.007 	}
    			& \textbf{0.064$\pm$0.008 	} \\
    \hline
    	vehicle
    	    &AB.ECC	
    			& 0.237$\pm$0.012 	
    			& 0.176$\pm$0.018 	
    			& 0.004$\pm$0.002 	
    			& \textbf{0.301$\pm$0.017 	}
    			& \textbf{0.297$\pm$0.021 	}
    			& 0.276$\pm$0.026 	 \\
    	    &TC.ECC	
    			& \textbf{0.196$\pm$0.015 	}
    			& \textbf{0.125$\pm$0.011 	}
    			& \textbf{0$\pm$0		}
    			& 0.312$\pm$0.029 	
    			& 0.313$\pm$0.027 	
    			& \textbf{0.272$\pm$0.037 	} \\
    \hline
    \end{tabular}
    \end{table*}
    \setlength{\tabcolsep}{1.4pt}

    \begin{figure*}[ht!]
    \begin{center}
        \subfigure[]
        {
            \includegraphics[width=0.28\textwidth,clip]{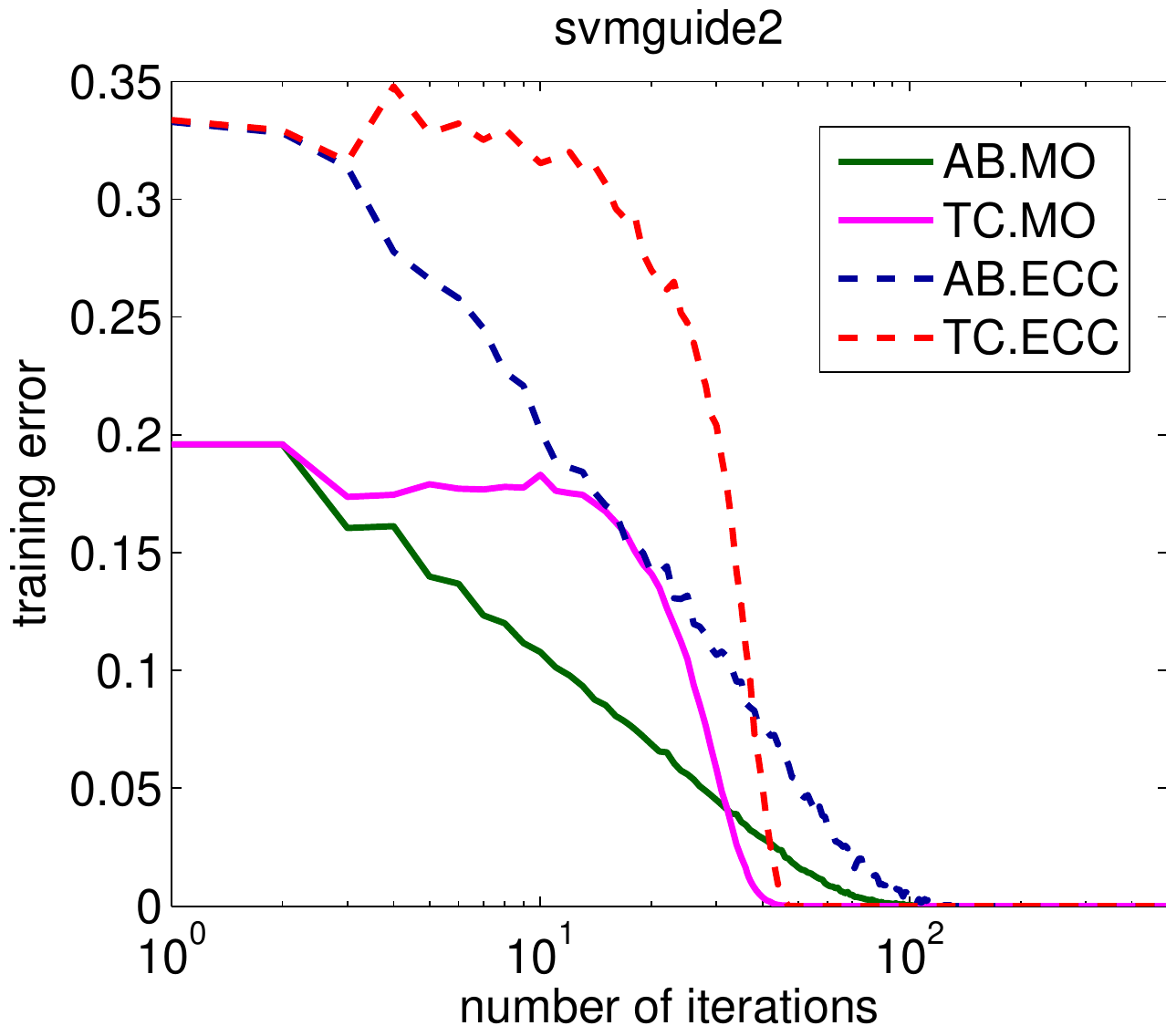}
            \label{fig:train1}
        }
        \subfigure[]
        {
            \includegraphics[width=0.28\textwidth,clip]{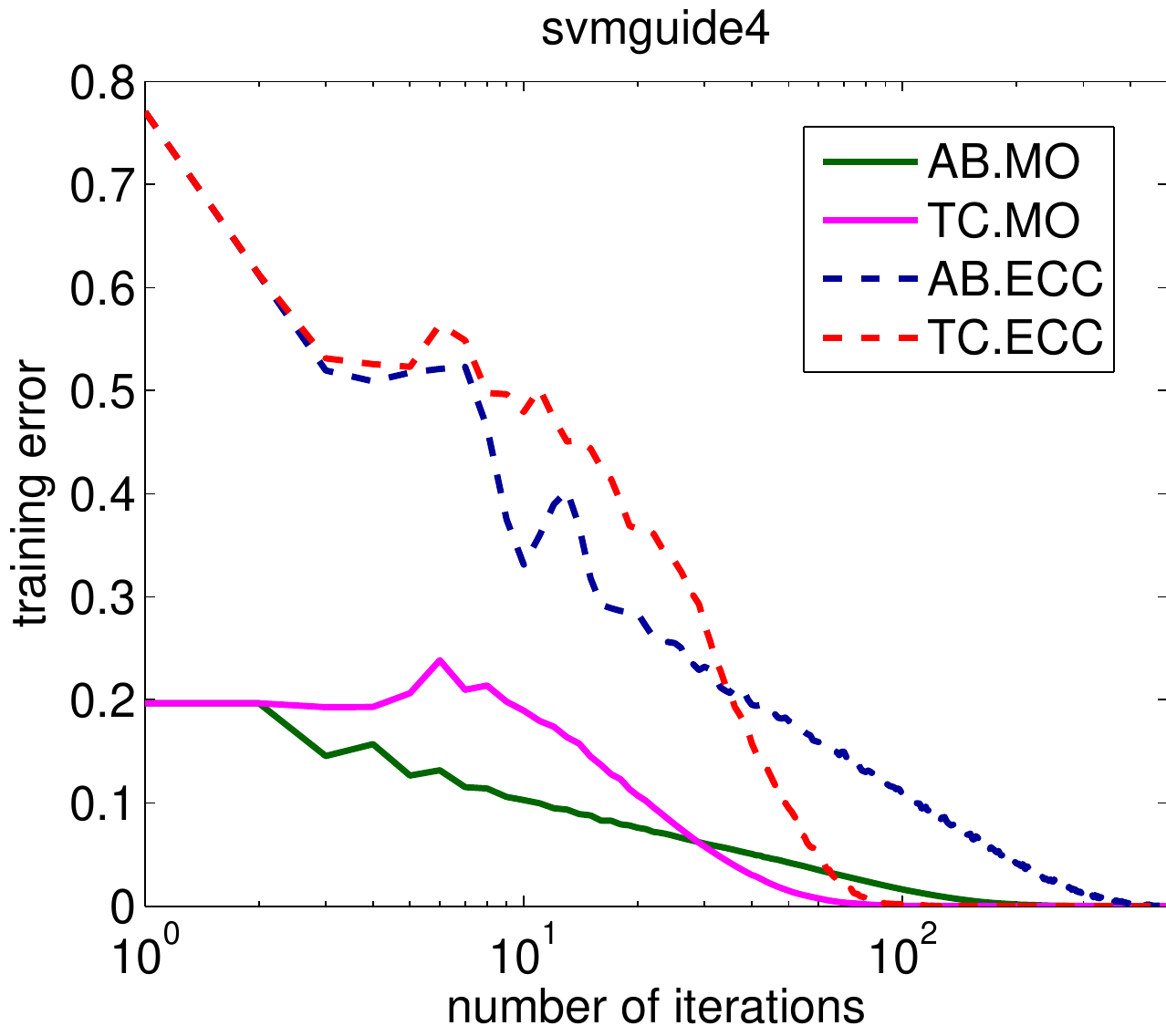}
            \label{fig:train2}
        }
        \subfigure[]
        {
            \includegraphics[width=0.28\textwidth,clip]{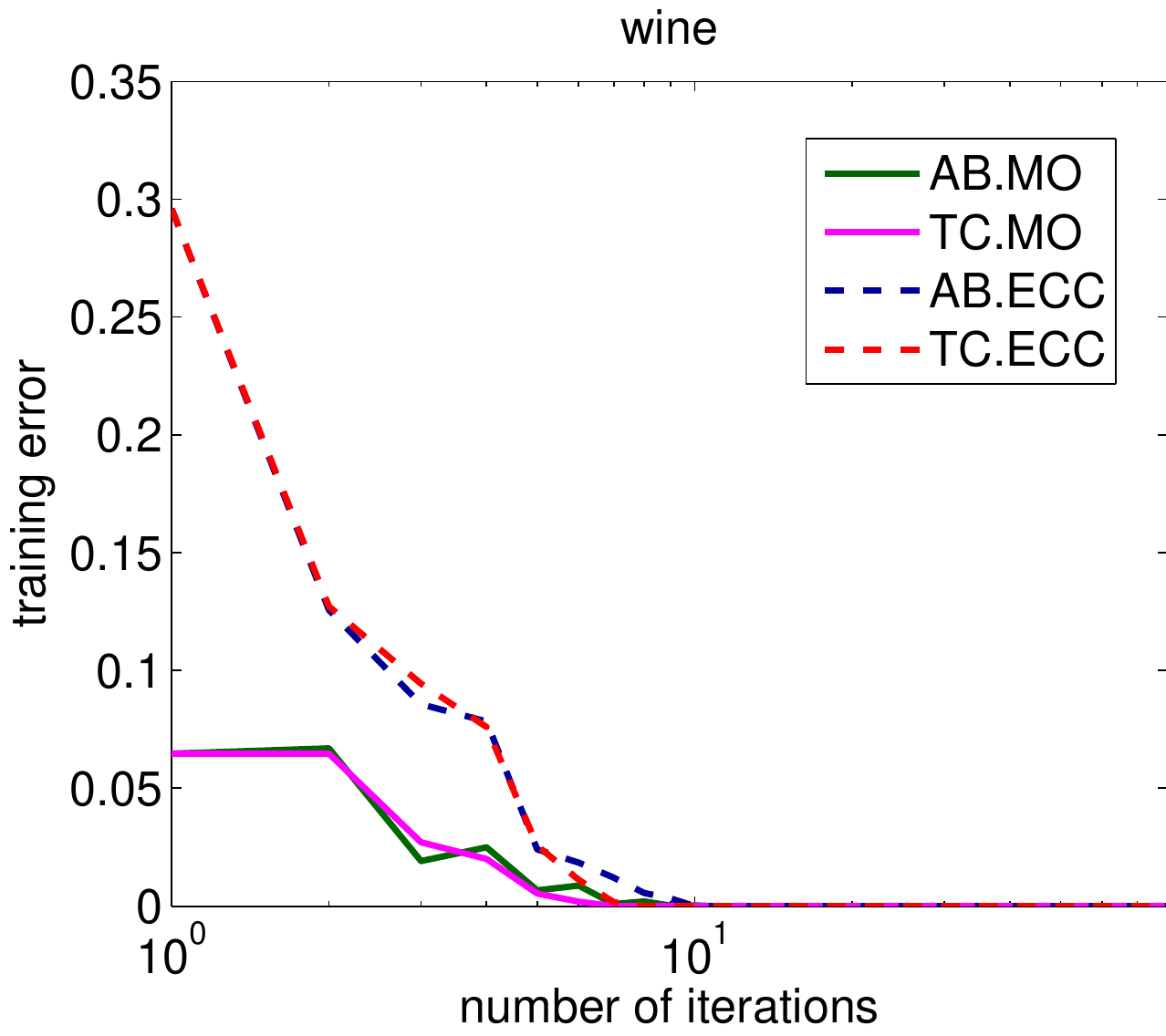}
            \label{fig:train3}
        }
        \subfigure[]
        {
            \includegraphics[width=0.28\textwidth,clip]{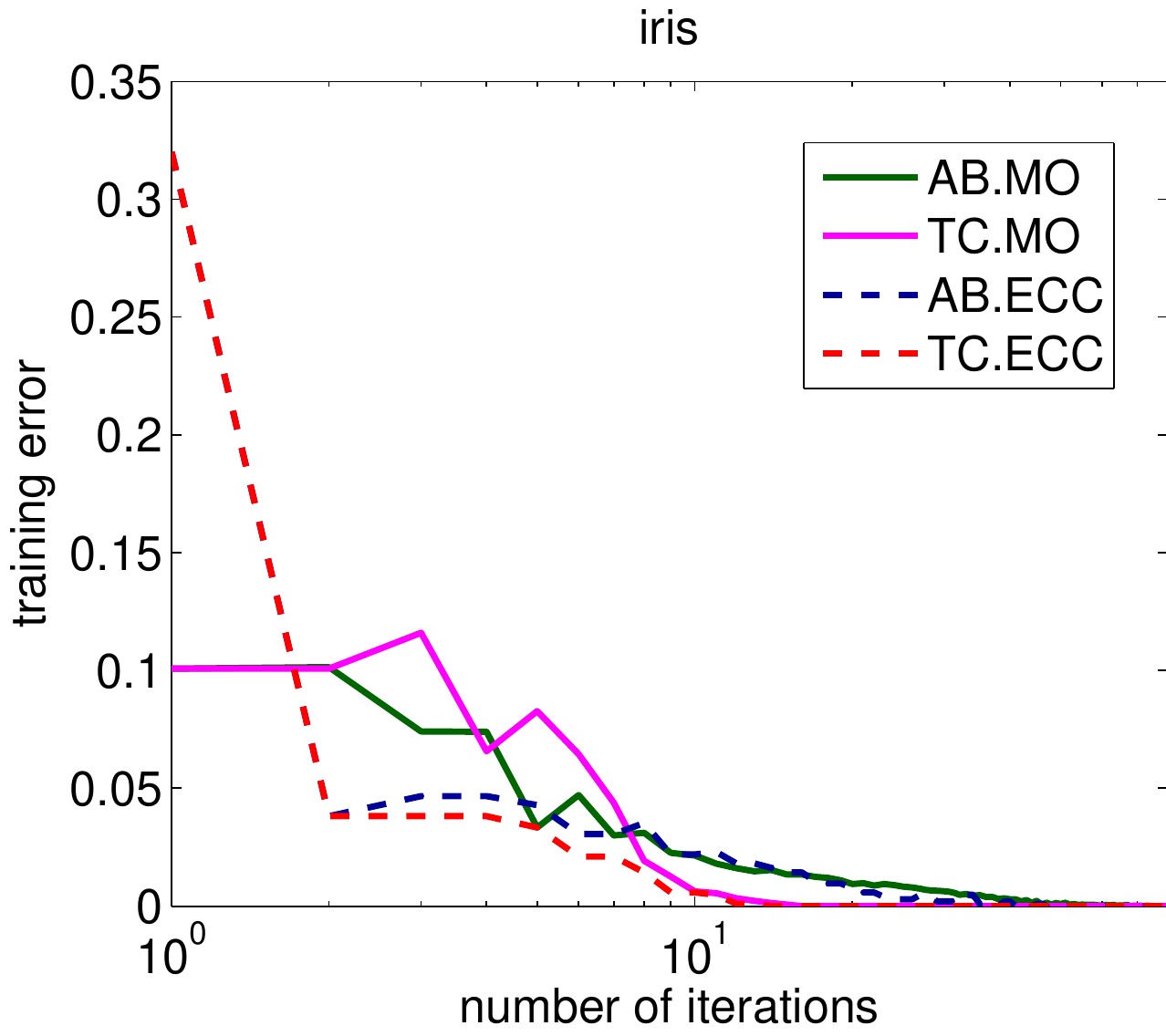}
            \label{fig:train4}
        }
        \subfigure[]
        {
            \includegraphics[width=0.28\textwidth,clip]{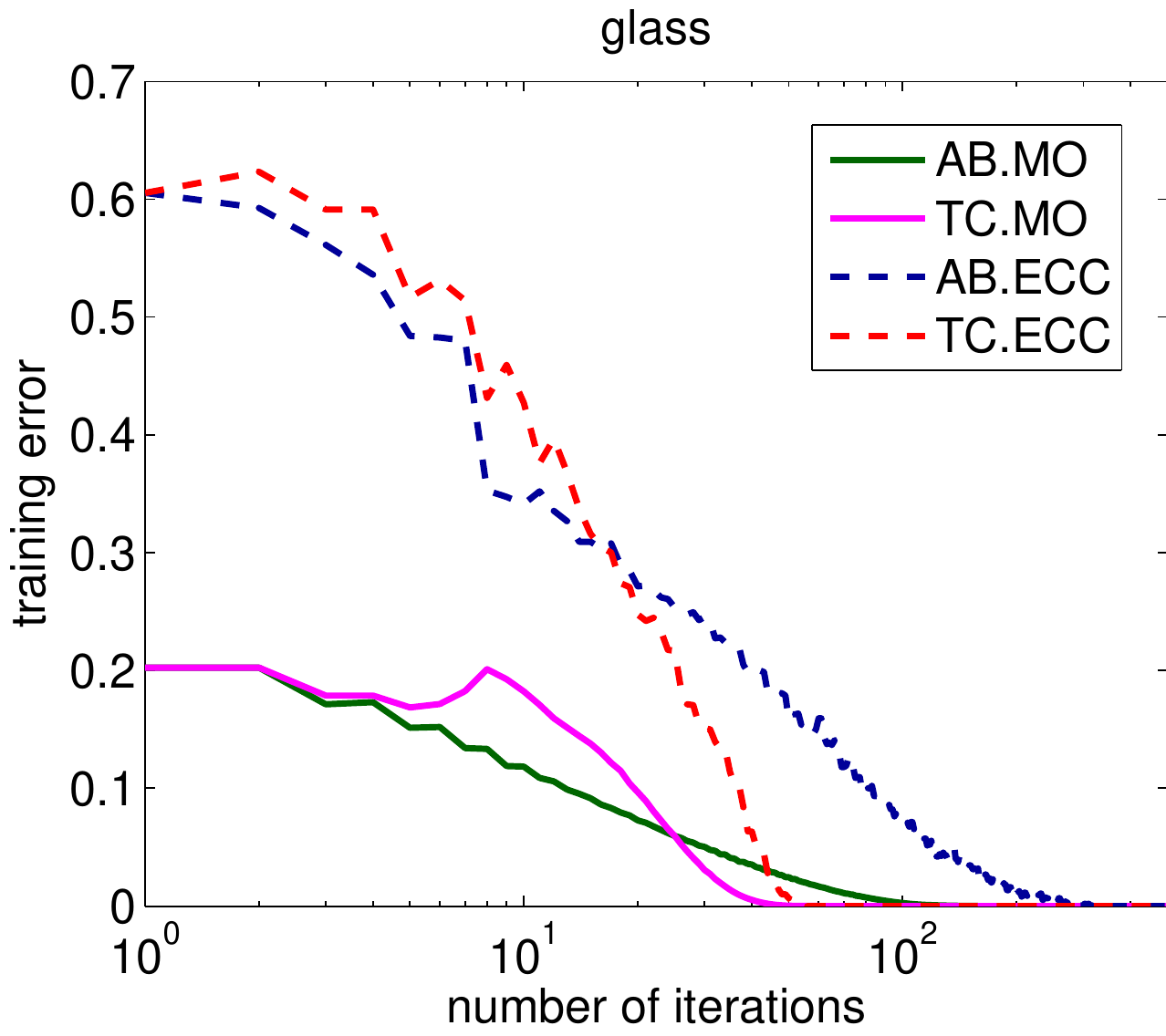}
            \label{fig:train5}
        }
        \subfigure[]
        {
            \includegraphics[width=0.28\textwidth,clip]{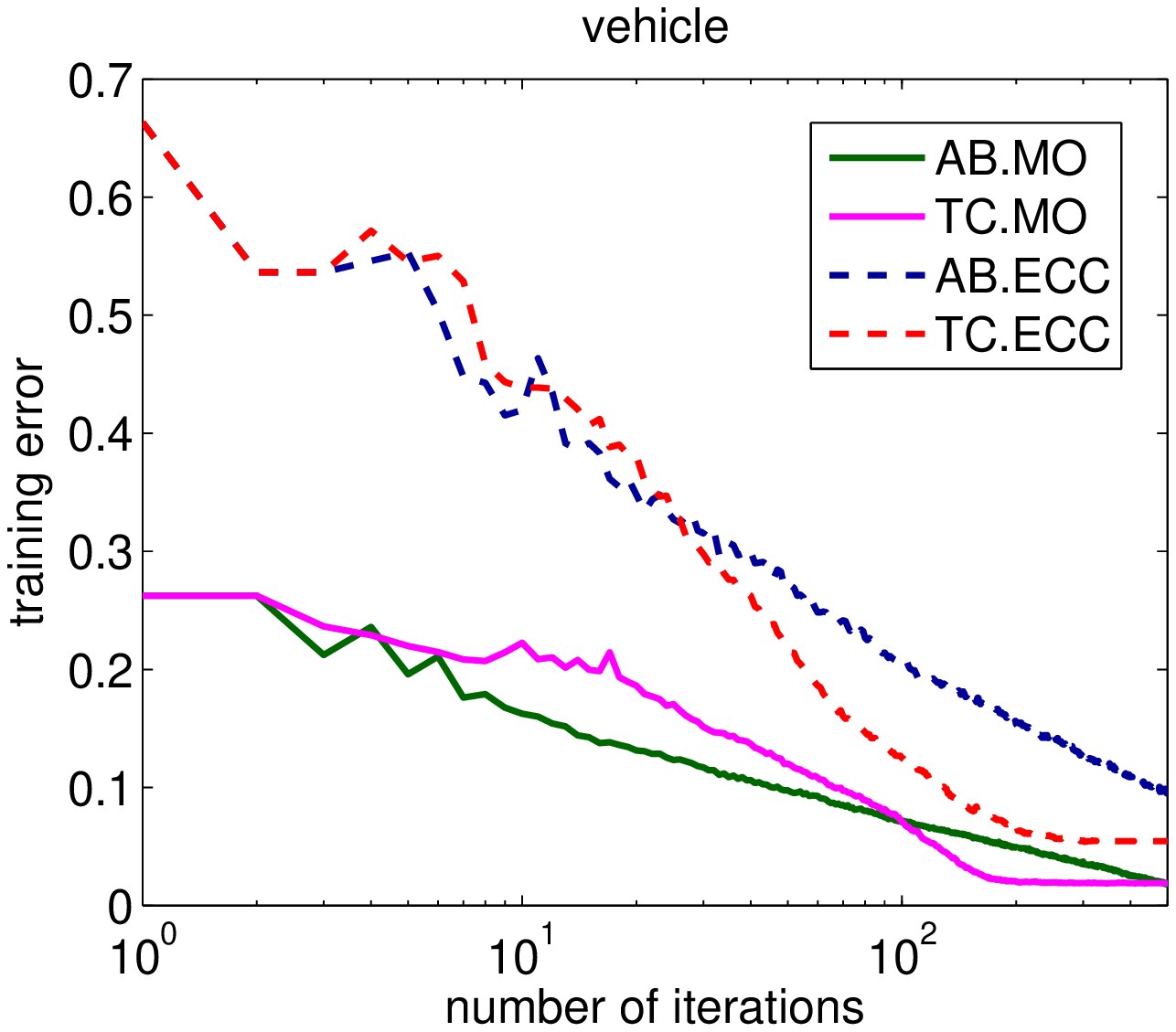}
            \label{fig:train6}
        }
    \end{center}
    \caption
    {
        Training error curves of AdaBoost.MO, \MultiA, AdaBoost.ECC and \MultiB{}
        on \svt, \svf, \win, \iri, \gla{} and \veh.
        The number of training iterations is $500$.
        Base learners are decision stumps.
	}
    \label{fig:train}
    \end{figure*}

    To further verify the generalization capability of our multiclass boosting algorithms,
    we run the Wilcoxon rank-sum test \cite{wilcoxon1970critical} on test errors
    in Tables \ref{table:dstump} and \ref{table:lda}.
    Wilcoxon rank-sum test is a nonparametric statistical tool for assessing the hypothesis
    that two sets of samples are drawn from the identical distribution.
    If the totally corrective boosting is comparable with its counterpart
    in classification error,
    the Wilcoxon test is supposed to output a higher significant probability.
    The results are reported in Table \ref{table:wilcoxon}.
    We can see the probabilities are higher enough ($> 0.8$) to claim the identity,
    except in the case ECC algorithms with decision stumps when $T = 50$
    and ECC algorithms with LDAs when $T = 500$.
    However, if we take a close look at those two cases,
    we can find where our totally corrective algorithms perform better than their counterparts.

    \setlength{\tabcolsep}{8pt}
    \begin{table}[!ht]
    \centering
    \caption{Wilconxon rank-sum test on classification errors}
    \label{table:wilcoxon}
    \begin{tabular}{l|ccc}
    \hline
    	algorithms & $T = 50$ & $T = 100$ & $T = 500$ \\
    \hline
    	MOs\ \ with stumps	& 0.902	& 0.878	& 1.000	\\
    	ECCs with stumps	& 0.743	& 0.821	& 0.983	\\
    	ECCs with LDAs	& 0.959	& 1.000	& 0.798	\\
    \hline
    \end{tabular}
    \end{table}
    \setlength{\tabcolsep}{1.4pt}

    Next, we investigate the minimum margin of training examples,
    which has a close relationship with the generalization error \cite{schapire1998boosting}.
    Warmuth and R\"{a}tsch \cite{warmuth2006totally} have proven that
    by introducing a slack variable $r$,
    totally corrective boosting can realize a larger margin than corrective version
    with the same number of weak hypotheses.
    We test this conclusion on datasets \svt, \svf{} and \iri.
    At each iteration, we record the minimum margin of training examples
    on the current combination of weak hypotheses.
    The results are illustrated in Fig. \ref{fig:margin}.
    It should be noted in algorithms of MOs and ECCs, the definitions of margin are different:
    \eqref{EQ:MO-marg} and \eqref{EQ:ECC-marg} respectively.
    However, it is clear that in any case, totally corrective boosting algorithms increase
    the margin much faster than the two previous ones.


    \begin{figure*}[!ht]
    \begin{center}
        \subfigure[]
        {
            \includegraphics[width=0.28\textwidth,clip]{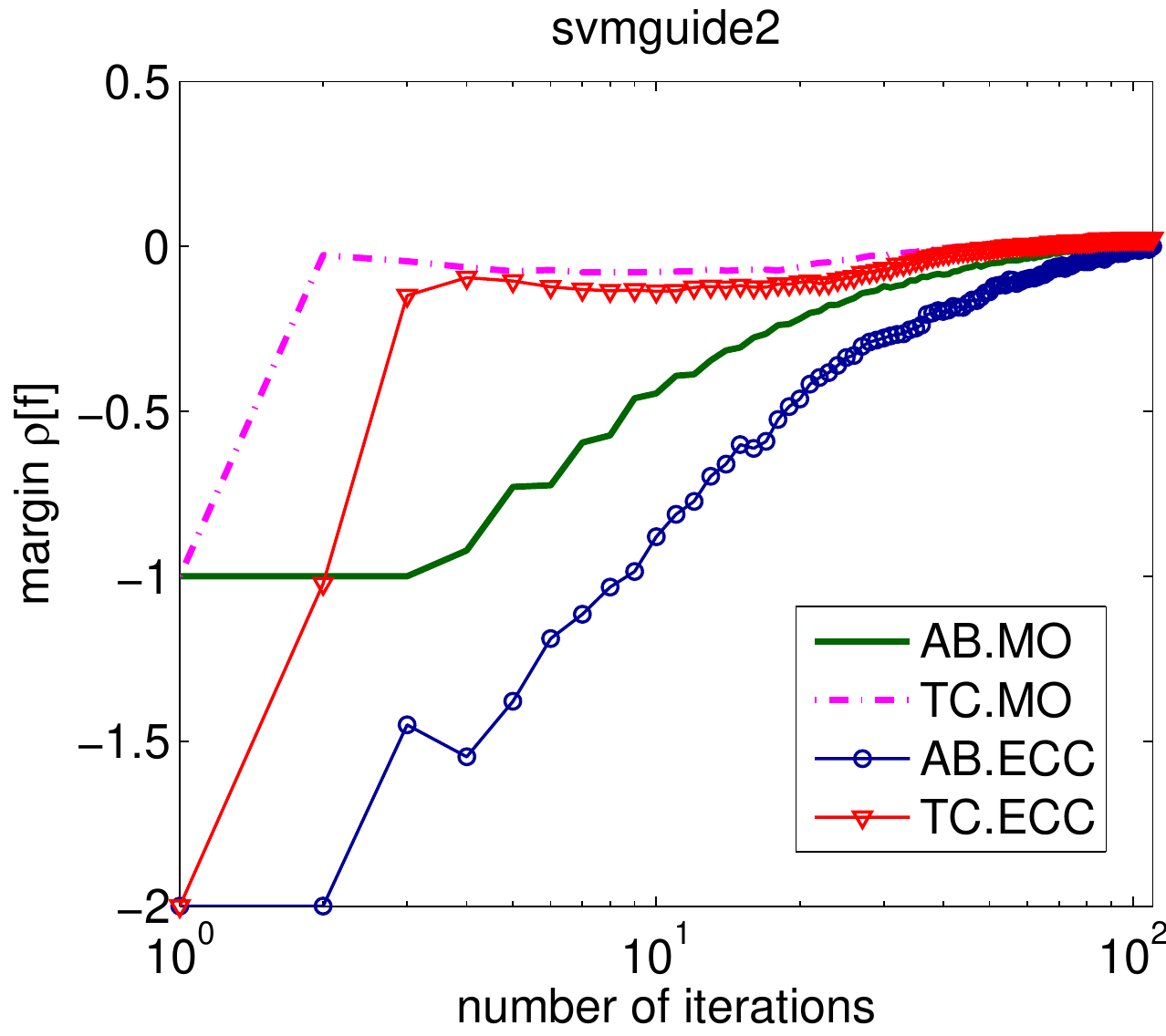}
            \label{fig:margin1}
        }
        \subfigure[]
        {
            \includegraphics[width=0.28\textwidth,clip]{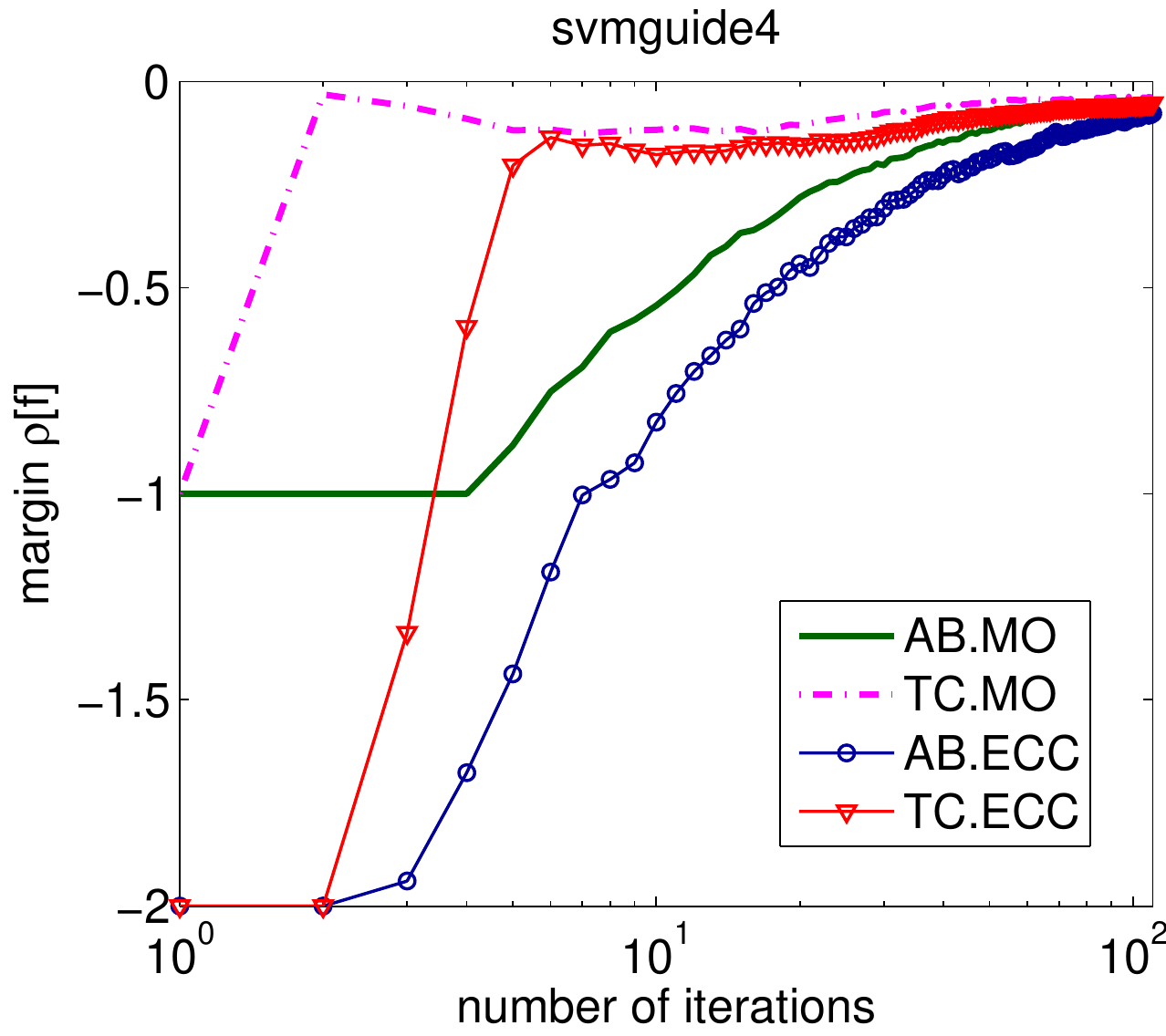}
            \label{fig:margin2}
        }
        \subfigure[]
        {
            \includegraphics[width=0.28\textwidth,clip]{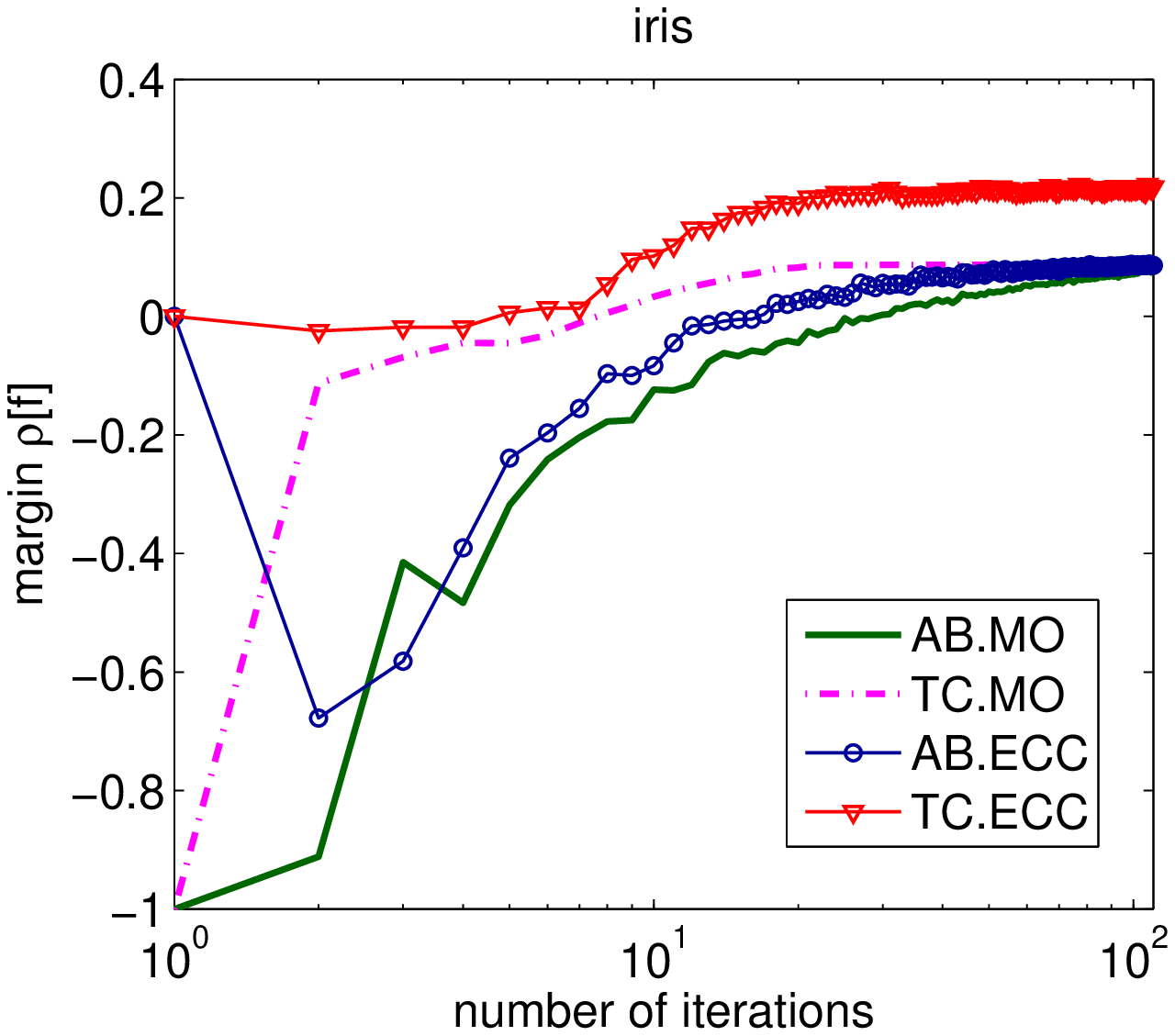}
            \label{fig:margin3}
        }
    \end{center}
    \caption
    {
        The minimum margin on assembled classifiers of AdaBoost.MO, \MultiA,
        AdaBoost.ECC and \MultiB.
        The definitions of margin are different in MOs and ECCs,
        however, it clearly shows that
        the totally corrective algorithms realize a larger margin than their counterparts
        within the same iterations.
	}
    \label{fig:margin}
    \end{figure*}

\section{Discussion and conclusion} \label{Sect:cons}

    We have presented two boosting algorithms for multiclass learning,
    which are mainly based on derivations of the Lagrange dual problems
    for AdaBoost.MO and AdaBoost.ECC.
    Using the column generation technique,
    we design new totally corrective boosting algorithms.
    The two algorithms can be formulated into a general framework base on the concept of margins.
    Actually, this framework can also incorporate other multiclass boosting algorithms,
    such as SAMME. In this paper, however,
    we have focused on multiclass boosting with binary weak learners.

    Furthermore, we indicate that the proposed boosting algorithms are totally corrective.
    This is the first time to use this concept in multiclass boosting learning.
    We also discuss the reason of introducing slack variables.
    Experiments on UCI datasets show that our new algorithms are much faster
    than their gradient descent counterparts in terms of convergence speed,
    but comparable with them in classification capability.
    The experimental results also demonstrate that totally corrective algorithms
    can maximize the example margin more aggressively.

\bibliographystyle{ieee}
\bibliography{mcbbib}

\begin{IEEEbiography}
    {Zhihui Hao}
        is a Ph.D. student in the School of Automation at Beijing Institute of Technology,
        China and currently visiting the Australian National University
        and NICTA, Canberra Research Laboratory, Australia.
        He received the B.E. degree in automation
        at Beijing Institute of Technology in 2006.
        His research interests include object detection, tracking,
        and machine learning in computer vision.
\end{IEEEbiography}

\begin{IEEEbiography}
    {Chunhua Shen}
        completed the Ph.D. degree from School of Computer Science,
        University of Adelaide, Australia in 2005;
        and the M.Phil. degree from Mathematical Sciences Institute,
        Australian National University, Australia in 2009.
        Since Oct. 2005, he has been working with the computer vision program,
        NICTA (National ICT Australia), Canberra Research Laboratory,
        where he is a senior research fellow and holds a continuing research position.
        His main research interests include statistical machine learning
        and its applications in computer vision and image processing.
\end{IEEEbiography}

\begin{IEEEbiography}
{Nick Barnes} completed a Ph.D. in 1999 at the University of Melbourne. In
1999 he was a visiting researcher at the LIRA-Lab, Genoa, Italy. From 2000 to
2003, he lectured in Computer Science and Software Engineering, at the University
of Melbourne. He is now a principal researcher at NICTA Canberra Research Laboratory
in computer vision.

\end{IEEEbiography}

\begin{IEEEbiography}
    {Bo Wang}
        received the Ph.D. degree from the Department of Control Theory and Engineering
        at Beijing Institute of Technology (BIT).
        He is currently working in the School of Automation at BIT.
        His main research interests include
        information detection and state control for high-speed moving objects.
   \end{IEEEbiography}

\end{document}